\documentclass[10pt]{article}
\usepackage[preprint]{tmlr}

\usepackage{amsmath,amsfonts,bm}

\def\eqref#1{equation~\ref{#1}}
\def\Eqref#1{Equation~\ref{#1}}

\def\1{\bm{1}}

\DeclareMathAlphabet{\mathsfit}{\encodingdefault}{\sfdefault}{m}{sl}
\SetMathAlphabet{\mathsfit}{bold}{\encodingdefault}{\sfdefault}{bx}{n}

\usepackage{hyperref}
\usepackage{url}
\usepackage{booktabs}
\usepackage{tabularx}
\usepackage[table]{xcolor}
\usepackage{makecell}
\usepackage{amsfonts}
\usepackage{nicefrac}
\usepackage{microtype}
\usepackage{natbib}
\usepackage{amsthm}
\usepackage{amsmath}
\usepackage{mathtools}
\usepackage{amssymb}

\newtheorem{theorem}{Theorem}
\newtheorem{lemma}{Lemma}
\newtheorem{definition}{Definition}
\newtheorem{corollary}{Corollary}
\newtheorem{remark}{Remark}

\newcommand{\esssup}{\operatorname{ess\,sup}}

\newcommand{\rev}[1]{#1}
\newenvironment{revision}{}{}

\title{Generalization Bounds via Vector-Valued $L_p$ Correlation $V_\alpha$}

\author{%
\begin{center}
\begin{tabular}[t]{@{}c@{\hspace{1.5em}}c@{\hspace{1.5em}}c@{}}
Mohammad Hossein & Mohammad Mahdi & Mohammad Hossein \\
Momeni Hamaneh & Mojahedian & Yassaee
\end{tabular}\\[0.65em]
{\normalfont\small
Information Systems and Security Laboratory (ISSL)}\\[0.15em]
{\normalfont\small
Sharif University of Technology, Tehran, Iran}\\[0.45em]
{\normalfont\small
\texttt{momenihamaneh1381@gmail.com}, \texttt{mojahedian@sharif.edu}, \texttt{yassaee@sharif.edu}}
\end{center}%
}

\begin{document}
\maketitle

\begingroup
\renewcommand{\thefootnote}{}
\footnotetext[0]{%
LLM-use disclosure: LLM-based tools were used to assist with manuscript language editing, code development, and debugging. All AI-assisted outputs were reviewed and verified by the authors, who take full responsibility for the scientific content, implementation, and reported results.

\smallskip
Code availability: The source code and experimental implementations are publicly available at \url{https://github.com/momenihamaneh1381/valpha-bounds}.
}
\endgroup

\begin{abstract}
Understanding generalization of randomized learning algorithms remains a central challenge. Mutual-information and total-variation bounds provide complementary ways to quantify data--hypothesis dependence, but they can behave quite differently in high-dimensional regimes. We study a one-parameter family of dependence measures, the vector-valued $L_p$-norm correlation $V_\alpha$, and derive generalization bounds whose dependence terms interpolate between total variation (TV) at $\alpha=1$ and variance and $\chi^2$ related behavior at $\alpha=2$. \rev{For $\alpha>2$, the H\"older-conjugate loss moment is below two, yielding heavy-tailed specializations that can remain finite even when the loss variance is infinite.}
A particularly interpretable specialization occurs at $\alpha=2$, where the bound can be expressed in terms of the variability of the algorithm's output density across datasets, $\int \sqrt{\mathrm{Var}_S[p(w\mid S)]}~ dw$. We show \rev{that the entire $V_\alpha$ family is invariant under smooth bijective reparameterizations of the hypothesis space.} We also introduce an adaptive density stability (ADS) condition that implies explicit control of $V_2$, and we verify it for Gibbs-type posteriors under bounded losses (yielding rates that depend on the inverse temperature and sample size). To address scaling, we extend the framework to a conditional formulation using supersamples, analogous in spirit to conditional mutual information (CMI), and discuss evaluations in the output space as a practical surrogate for estimating dependence. Simulations on Bayesian linear regression illustrate how tuning $\alpha$ can change the tightness of the resulting bounds across regimes. \rev{Additional supersample experiments on noisy MNIST and on a pretrained ResNet-50 for CIFAR-10 show similar behavior of the projected loss-space bounds beyond the original clean-MNIST setting.}
\end{abstract}

\section{Introduction}
\label{sec:introduction}

Understanding why large, over-parameterized models generalize despite their
capacity to memorize training data remains a central challenge in modern
machine learning~\citep{shalev2014understanding}. Information-theoretic
approaches address this question by relating generalization to the dependence
between the training sample \(S\) and the randomized output \(W\) of a learning
algorithm. A representative result of \citet{xu2017information} shows that,
under a subgaussian loss assumption, the expected generalization gap is
controlled by a mutual information (MI) term of the form
\(\sqrt{I(S;W)/n}\).

A large body of subsequent work has refined this basic principle. Conditional
mutual information (CMI) bounds use a supersample construction to avoid
measuring dependence directly between the full sample and the parameter vector
\citep{steinke2020}. Sample-wise and evaluated CMI variants further decompose
the dependence into lower-dimensional quantities
\citep{hellstrom2022new, wang2023sample}, while more general divergence-based
frameworks provide flexible decoupling inequalities
\citep{wang2024generalization}. These developments show that the choice of
dependence measure matters: different information quantities can lead to very
different tightness, dimensional behavior, and estimability.

In this paper, we study a one-parameter family of dependence measures based on
the vector-valued \(L_p\)-norm correlation \(V_\alpha\)
\citep{mojahedian2019correlation}. The parameter $\alpha \ge 1$ provides a tunable interpolation between the $L_1$ endpoint at $\alpha=1$, corresponding to twice the total variation (TV) distance, and the variance and $\chi^2$ related behavior at $\alpha=2$, while larger values of $\alpha$ place greater emphasis on large likelihood-ratio deviations. This produces a family of generalization bounds in which the loss-sensitivity term and the dependence term move in opposite directions as \(\alpha\) varies. \rev{For
\(\alpha>2\), the conjugate exponent \(q=\alpha/(\alpha-1)\) lies in \((1,2)\), so the bound can be paired with lower-order loss moments even when the variance is infinite.} Consequently, the best value of \(\alpha\) is not universal, but can depend on the algorithm, loss, and data distribution.

The case \(\alpha=2\) has a particularly transparent interpretation. In this
regime, the dependence term can be written in terms of the variability of the
algorithm's output density across training samples,
\(\int \sqrt{\mathrm{Var}_S[p(w|S)]}\,dw\). This connects generalization to a
density-level notion of algorithmic stability. We also show that this
\(V_2\)-based quantity is invariant under smooth bijective
reparameterizations of the hypothesis space, so the value of the bound is not
an artifact of the coordinate system used to represent \(W\).

Directly estimating parameter-space quantities such as \(V_\alpha(S;W)\) or
\(V_\alpha(U;W|\tilde S)\) can still be difficult in high-dimensional models.
To address this, we develop a supersample conditional formulation and then
derive projected and loss-space variants. These results do not imply that all
dimension dependence disappears: the conditional dependence term can still be
large if the learned object reveals substantial information about the selector.
Rather, the conditional and projected bounds shift the analysis to quantities
that can be lower-dimensional, more interpretable, and more amenable to
empirical estimation.

\paragraph{Contributions.} In summary, our primary contributions are:
\begin{enumerate}
    \item \textbf{A tunable \(V_\alpha\)-based generalization framework.}
    We derive a family of information-theoretic generalization bounds based on
    \(V_\alpha\). The bounds recover an \(L_1\)/TV-type endpoint at
    \(\alpha=1\), a variance-based specialization at \(\alpha=2\), and \rev{an
    \(n^{-1/2}\) loss-sensitivity prefactor under subgaussian loss assumptions for
    \(\alpha>1\). The full bound has \(n^{-1/2}\)-type decay when the associated
    dependence term remains uniformly bounded in \(n\).} \rev{For \(\alpha>2\), they also admit a heavy-tailed
    specialization requiring only a finite \(q\)-th centered loss moment,
    where \(q=\alpha/(\alpha-1)\in(1,2)\).}

    \item \textbf{A variance-based view of randomized algorithmic stability.}
    We isolate the \(\alpha=2\) case and express the dependence term through
    the variability of the output density \(p(w|S)\) across datasets. \rev{More
    generally, we prove that \(V_\alpha\) is invariant under smooth bijective
    reparameterizations of the hypothesis space for every \(\alpha\ge 1\).}

    \item \textbf{Adaptive Density Stability.}
    We introduce Adaptive Density Stability (ADS), a strong pointwise
    distributional stability condition that gives explicit control of
    \(V_2(S;W)\). For bounded losses, we verify ADS for fixed-temperature Gibbs
    posteriors of the form
    \(p(w|S)\propto \exp(-\beta L_S(w))\pi(w)\), obtaining
    \(\gamma_n=\mathcal{O}(1/n)\) when \(\beta\) and the loss bound are fixed.

    \item \textbf{Conditional, projected, and loss-space \(V_\alpha\) bounds.}
    We extend the framework to a supersample formulation analogous in spirit
    to CMI. The resulting bounds measure dependence between the selector
    \(U\) and the learned object conditional on the supersample \(\tilde S\).
    We further derive projected output-space, loss-difference, loss-pair, and
    binarized variants that can reduce the effective dimension of the
    dependence term when the chosen statistic preserves the relevant
    selected-versus-ghost gap.

    \item \textbf{Empirical evaluation of the resulting estimators.}
    We evaluate the proposed bounds and plug-in estimators on a supersample
    MNIST experiment\rev{, on a pretrained
    ResNet-50 for CIFAR-10,} and on Bayesian linear regression simulations. The
    experiments illustrate how the choice of \(\alpha\), the statistic used for
    projection, and binarization affect the empirical tightness and
    estimability of the resulting \(V_\alpha\)-based quantities.
\end{enumerate}

\section{Preliminaries and Related Work}
\label{sec:preliminaries}

In this section, we establish the foundational concepts from statistical learning theory, contextualize our approach within the recent literature on information-theoretic bounds, and introduce the $V_\alpha$ correlation measure.

\subsection{Statistical Learning Setup}
We follow the standard statistical learning framework. Let $\mathcal{Z} = \mathcal{X} \times \mathcal{Y}$ be the instance space and $\mathcal{W}$ be the hypothesis space. A learning algorithm maps a training sample $S = (Z_1, \dots, Z_n)$, drawn i.i.d.\ from an unknown data-generating distribution $\mu$, to a hypothesis $W \in \mathcal{W}$. This mapping is often randomized and is described by a conditional probability distribution $P_{W|S}$. The joint distribution over hypotheses and samples is $P_{W,S} = \mu^{\otimes n} \otimes P_{W|S}$.

The performance of a hypothesis $w \in \mathcal{W}$ is measured by its population risk with respect to a loss function $\ell(w,z)$: $L_{\mu}(w) \coloneqq \mathbb{E}_{Z \sim \mu}[\ell(w,Z)]$. Because $\mu$ is unknown, the algorithm minimizes the empirical risk $L_{S}(w) \coloneqq \frac{1}{n}\sum_{i=1}^{n}\ell(w,Z_{i})$. The generalization gap is defined as $\mathrm{gen}(W, S) \coloneqq L_{\mu}(W) - L_{S}(W)$. 
We consider the expected generalization error $\overline{\mathrm{gen}}\left(\mu, P_{W|S}\right) \coloneqq \mathbb{E}_{P_{W,S}}\left[L_{\mu}(W) - L_{S}(W)\right],$
and in this work we focus on bounding its absolute value, $\left|\overline{\mathrm{gen}}\left(\mu, P_{W|S}\right)\right|$.
Our analytical decoupling relies on the notion of subgaussianity, which bounds the tail probabilities of random variables.
\begin{definition}[\citet{vershynin2018high}]
	\label{def:subgaussian}
	A zero-mean random variable $X$ is $\sigma^2$-subgaussian if its moment generating function satisfies $\mathbb{E}\left[e^{\lambda X}\right] \le e^{\lambda^2\sigma^2/2}$ for all $\lambda \in \mathbb{R}$. This implies a variance bound $\mathrm{Var}[X] \le \sigma^2$. By Hoeffding's Lemma, any random variable bounded almost surely in $[a,b]$ is $\frac{(b-a)^2}{4}$-subgaussian.
\end{definition}

\begin{definition}
	The $L_\alpha$-norm of a random variable $X$ is defined as $\|X\|_\alpha = \left( \mathbb{E}\bigl[ |X|^\alpha \bigr] \right)^{1/\alpha}$ for $\alpha \ge 1$. The $L_\alpha$-norm is an increasing function of $\alpha$ and is convex in $X$ (and strictly convex for $\alpha>1$).
\end{definition}

\subsection{Information-Theoretic Bounds and Algorithmic Stability}

The concept of algorithmic stability formalizes the intuition that algorithms insensitive to small changes in the training data generalize well~\citep{bousquet2002stability}. Information-theoretic generalization bounds quantify this stability using dependence measures between the data $S$ and the weights $W$. A seminal result by \citet{xu2017information} demonstrated that if $\ell(w,Z)$ is $\sigma^{2}$-subgaussian, then $\left|\overline{\mathrm{gen}}(\mu, P_{W|S})\right| \le \sqrt{\frac{2\sigma^{2}}{n}I(W;S)}$.

\begin{revision}
\paragraph{Relation to PAC-Bayes.}
PAC-Bayes gives a complementary analysis of randomized predictors, including
Gibbs posteriors, through a data-dependent posterior $Q_S$ and a fixed prior
$\pi$~\citep{mcallester,seeger2002pac,catoni2007pac,germain2016pac}. Its
complexity term is typically $\mathrm{KL}(Q_S\Vert\pi)$, whereas our
unconditional dependence quantities compare $P_{W\mid S}$ with the
algorithm-induced marginal $P_W$. These reference choices satisfy
\begin{align}
\mathbb E_S\!\left[\mathrm{KL}(P_{W\mid S}\Vert\pi)\right]
= I(S;W)+\mathrm{KL}(P_W\Vert\pi).
\label{eq:pacbayes_reference_identity}
\end{align}
Thus the two frameworks are complementary rather than exact specializations of
one another: PAC-Bayes typically provides prior-relative, high-probability risk
bounds, while the present results target expected generalization gaps and also
admit conditional and projected dependence measures. A numerical ordering for
Gibbs posteriors cannot be determined from $\beta$ and the loss range alone,
since the PAC-Bayes side additionally depends on the prior-relative KL term.
\end{revision}

\textbf{Scaling to High Dimensions.} A known limitation of the classical MI bound is that $I(S;W)$ scales linearly with the number of parameters, making it vacuous for deep networks. To resolve this, \citet{steinke2020} introduced the CMI bound using a supersample framework, shifting the dependence to the sample size $n$. Recent advancements have further tightened these guarantees using sample-wise evaluations~\citep{hellstrom2022new, wang2023sample} and $f$-divergence frameworks~\citep{wang2024generalization}. In Section~\ref{sec:scaling_high_dim}, we adapt the supersample machinery to our $V_\alpha$ measure, ensuring our framework is similarly equipped for high-dimensional regimes (see Appendix~\ref{app:extended_related_work} for extended discussions).

\textbf{Variance-Based Perspectives.} Our work strongly intersects with recent research connecting generalization to variance. \citet{gastpar2025algorithms} recently proved that an algorithm exhibits tight generalization if and only if the conditional variance of its population loss, $\mathrm{Var}(L_\mu(A(S))|S)$, is small. While their work analyzes variance in the \textit{outcome (loss) space}, our framework analyzes variance directly in the \textit{hypothesis space}. As we will show in Section~\ref{sec:special_cases}, our $\alpha=2$ bound is controlled by $\mathrm{Var}_S[p(w|S)]$. This offers a complementary, distribution-centric definition of stability that is invariant under reparameterization.

\subsection{The \texorpdfstring{$V_\alpha$}{V-alpha} Correlation Measure}
To unify these perspectives, our framework is built upon the vector-valued $L_p$-norm correlation measure $V_\alpha(A;B)$, introduced by \citet{mojahedian2019correlation}. For a joint distribution $p_{AB}$ and marginals $p_A, p_B$, the measure is defined for any $\alpha \ge 1$ as:
\begin{align}
    \label{eq:v_alpha_def}
    V_\alpha(A; B) \coloneqq \mathbb{E}_{b \sim p_B}\left[ \left( \mathbb{E}_{a \sim p_A}\left| \frac{p_{B|A}(b|a)}{p_B(b)} - 1 \right|^\alpha \right)^{1/\alpha} \right].
\end{align}
The measure is non-decreasing in $\alpha$ and provides a tunable bridge between well-known dependence measures. For $\alpha=1$, \rev{it equals the $L_1$ distance, i.e., twice the total variation distance: $V_1(A;B) = \|p_{AB} - p_A p_B\|_1 = 2\,\mathrm{TV}(P_{AB},P_AP_B)$.} For $\alpha > 1$, it acts as an exponential divergence metric closely related to the R\'enyi mutual information of order $\alpha$ \citep{mojahedian2019correlation}.

A key property of $V_\alpha$, which is central to its analytical tractability, is its convexity. It can be expressed as a sum (or integral) of the $L_\alpha$-norms of the conditional density deviations:
\begin{align}
    V_{\alpha}(A; B) = \int_{\mathcal{B}} \left\| p_{B|A}(b|\cdot) - p_B(b) \right\|_{\alpha, p_A} db.
\end{align}
Since the $L_\alpha$-norm is convex for $\alpha \ge 1$, \rev{$V_\alpha(A;B)$ is convex with respect to the conditional distributions.} This property enables optimization-based interpretations and facilitates the decoupling analysis required for generalization bounds.



\section{The \texorpdfstring{$V_{\alpha}$}{V-alpha} Information Framework}
\label{sec:main_result}

Having established the necessary preliminaries, we now present our main
\(V_\alpha\)-based decoupling framework. We first bound the absolute expected
generalization error by (i) a \emph{loss-sensitivity} term and (ii) a
\emph{dependence} term measured by \(V_\alpha\). We then discuss how, under
standard subgaussian assumptions, \rev{the explicit loss-sensitivity prefactor scales as
\(\mathcal O(n^{-1/2})\) for any \(\alpha>1\); the complete bound inherits this rate whenever the corresponding dependence term remains \(\mathcal O(1)\) in \(n\).}

\paragraph{A note on dimensionality.}
The bounds in this section evaluate stability directly in the hypothesis space
\(\mathcal W\). This is conceptually useful, but in over-parameterized models
parameter-space dependence measures can be very large and difficult to estimate.
In Section~\ref{sec:scaling_high_dim} we introduce a scalable \emph{conditional}
supersample formulation that measures dependence in the selector space rather
than directly in \(\mathcal W\).

\subsection{A general bound on the generalization error}
\label{subsection:gen_bound}

Our main theorem gives a one-parameter family of bounds on the absolute expected
generalization error. The proof is a standard change-of-measure plus
H\"older-decoupling argument and is given in Appendix~\ref{proof:th-main}.

\begin{theorem}[General \(V_\alpha\) bound]
\label{th-main}
For any learning algorithm and any measurable loss for which the displayed
quantities are well-defined, for every \(\alpha\ge 1\),
\begin{align}
\Bigl|
\overline{\mathrm{gen}}(\mu,P_{W\mid S})
\Bigr|
\;\le\;
\sup_{w\in\mathcal W}
\bigl\|L_S(w)-L_\mu(w)\bigr\|_{\alpha',\,S}\;
V_\alpha(S;W),
\label{eq:thm1_main}
\end{align}
where \(\alpha'\) is the H\"older conjugate of \(\alpha\) (i.e.,
\(1/\alpha+1/\alpha'=1\), with \(\alpha'=\infty\) when \(\alpha=1\)). Here
\(\|\cdot\|_{\alpha',\,S}\) is the \(L_{\alpha'}\) norm with respect to the law of
\(S\sim\mu^{\otimes n}\).
\end{theorem}

\begin{revision}
Theorem~\ref{th-main} should therefore be viewed as the master
one-parameter decoupling bound: the cases $\alpha=1$ and $\alpha=2$
are particularly interpretable reference regimes within the same family,
rather than separate analyses, while intermediate values trade off the
loss-sensitivity and dependence terms.
\end{revision}

\paragraph{Symmetric form.}
A symmetric variant obtained by swapping the roles of \(S\) and \(W\) is stated
and proved in Appendix~\ref{app:symmetry_bound}.

\subsection{Subgaussian specialization (rate form)}
\label{subsection:subgaussian_rate}

Assume that for all \(w\in\mathcal W\), the single-sample loss \(\ell(w,Z)\) is
\(\sigma^2\)-subgaussian under \(Z\sim\mu\). Then for any \(\alpha>1\), the
empirical risk deviation \(L_S(w)-L_\mu(w)\) is \((\sigma^2/n)\)-subgaussian, and
hence admits moment control of the form
\begin{align}
\sup_{w\in\mathcal W}\bigl\|L_S(w)-L_\mu(w)\bigr\|_{\alpha',\,S}
\;\lesssim\;
\sigma\sqrt{\frac{\alpha'}{n}}.
\label{eq:subg_rate_form}
\end{align}
Substituting \eqref{eq:subg_rate_form} into Theorem~\ref{th-main} yields the rate
form
\begin{align}
\Bigl|\overline{\mathrm{gen}}(\mu,P_{W\mid S})\Bigr|
\;\lesssim\;
\sigma\sqrt{\frac{\alpha'}{n}}\;V_\alpha(S;W),
\qquad \alpha>1.
\label{eq:subg_rate_form_gen}
\end{align}
An explicit constant (involving a Gamma-function expression) is stated and
proved in Appendix~\ref{app:subgaussian_corollary}. As \(\alpha\) varies, the two
factors in \eqref{eq:subg_rate_form_gen} move in opposite directions: \(V_\alpha\)
is non-decreasing in \(\alpha\) \citep{mojahedian2019correlation}, whereas the
loss-sensitivity factor depends on \(\alpha'\) and decreases as \(\alpha\)
increases.

\rev{The \(n^{-1/2}\) factor here is the explicit loss-moment contribution. The
complete right-hand side in \eqref{eq:subg_rate_form_gen} scales as
\(\mathcal O(n^{-1/2})\) only under corresponding control of
\(V_\alpha(S;W)\), for example when it is uniformly \(\mathcal O(1)\) in
\(n\); if the dependence term itself decreases with \(n\), the resulting rate
can be faster.}

\subsection{Recovering \texorpdfstring{$n^{-1/2}$}{n-to-the-minus-one-half} rates and the \texorpdfstring{$\alpha=1$}{alpha=1} regime}
\label{subsection:alpha_1_rates}

A critical requirement for any generalization bound is that it decays with
sample size \(n\), typically at \(\mathcal O(n^{-1/2})\). A direct application
of Theorem~\ref{th-main} at \(\alpha=1\) (so \(\alpha'=\infty\)) yields a
worst-case loss deviation:
\begin{align}
\Bigl|\overline{\mathrm{gen}}(\mu,P_{W\mid S})\Bigr|
&\le
\sup_{w\in\mathcal W}
\Bigl\{
\esssup_{S\sim\mu^{\otimes n}}
\bigl|L_S(w)-L_\mu(w)\bigr|
\Bigr\}
\cdot V_1(S;W).
\label{eq:alpha1_bound}
\end{align}
Here \(V_1(S;W)=\|P_{W,S}-P_WP_S\|_1 = 2\,\mathrm{TV}(P_{W,S},P_WP_S)\). By
Pinsker's inequality,
\(\|P_{W,S}-P_WP_S\|_1 \le \sqrt{2I(S;W)}\), so one may further upper-bound the
dependence term using mutual information. However, the \(L_\infty\) loss term in
\eqref{eq:alpha1_bound} is a worst-case deviation over datasets and is often
\(\mathcal O(1)\), rather than decaying with \(n\).

This lack of an \(\mathcal O(n^{-1/2})\) rate at \(\alpha=1\) is an artifact of
the strict worst-case \(L_\infty\) limit. In contrast, for every \(\alpha>1\),
the conjugate \(\alpha'<\infty\) and subgaussianity yields moment control of
\(L_S(w)-L_\mu(w)\), \rev{producing the \(\mathcal O(n^{-1/2})\) loss-sensitivity
factor in \eqref{eq:subg_rate_form_gen}; the rate of the complete bound also
depends on how \(V_\alpha(S;W)\) varies with \(n\).} Thus, by moving away from the \(L^1/L^\infty\)
endpoint and choosing \(\alpha>1\) (\rev{including} \(\alpha=2\) in
Section~\ref{sec:special_cases}), we retain probability-space dependence control
while recovering fast statistical rates.

\begin{revision}
\paragraph{Relation to canonical mutual information.}
No value of \(\alpha\) in Theorem~\ref{th-main} exactly reproduces the
canonical MI bound. The \(V_\alpha\) family controls an
\(L_\alpha\) norm of the centered likelihood ratio, whereas mutual information
averages a logarithm of the likelihood ratio. At \(\alpha=1\), \(V_1\) is the
\(L_1\) distance (twice total variation), so Pinsker's inequality relates this
endpoint to \(I(S;W)\), but the accompanying loss term remains the
\(L_\infty\) deviation in \eqref{eq:alpha1_bound}. Thus the \(V_\alpha\) and
canonical MI approaches are complementary rather than exact specializations of
one another.

\subsection{Heavy-tailed example: finite lower-order moments}
\label{subsection:heavy_tailed}

The continuum beyond \(\alpha=2\) has a useful interpretation when the loss is
heavy-tailed. Let
\(q:=\alpha'=\alpha/(\alpha-1)\). For \(\alpha>2\), we have
\(q\in(1,2)\), so Theorem~\ref{th-main} can be controlled using a lower-order
moment even when a second moment does not exist.

\begin{corollary}[Heavy-tailed specialization]
\label{cor:heavy_tailed}
Fix \(\alpha>2\) and set
\(q=\alpha/(\alpha-1)\in(1,2)\). Suppose
\begin{align}
M_q
&:=
\sup_{w\in\mathcal W}
\left(
\mathbb E_{Z\sim\mu}
\left|\ell(w,Z)-L_\mu(w)\right|^q
\right)^{1/q}
<\infty.
\label{eq:heavy_tail_mq}
\end{align}
Then
\begin{align}
\left|
\overline{\mathrm{gen}}(\mu,P_{W\mid S})
\right|
\le
2^{1/q}\,n^{-1/\alpha}\,M_q\,V_\alpha(S;W).
\label{eq:heavy_tail_bound}
\end{align}
\end{corollary}

The proof, given in Appendix~\ref{app:heavy_tail_proof}, follows from the
von Bahr--Esseen moment inequality. Corollary~\ref{cor:heavy_tailed} remains
informative for losses with infinite variance provided the centered
\(q\)-th moment in \eqref{eq:heavy_tail_mq} is finite. The price for this weaker
moment assumption is the slower generic rate
\(n^{-1/\alpha}=n^{-(1-1/q)}\), instead of the \(n^{-1/2}\) rate associated
with finite-variance or subgaussian control. If either \(M_q\) or
\(V_\alpha(S;W)\) is infinite, the specialization is valid but vacuous.

\paragraph{The endpoint \(\alpha=\infty\).}
As \(\alpha\to\infty\), the conjugate exponent satisfies \(q=\alpha'\downarrow1\).
The corresponding dependence endpoint is
\begin{align}
V_\infty(S;W)
:=
\mathbb E_{W\sim P_W}
\left[
\esssup_{S\sim P_S}
\left|
\frac{p(W\mid S)}{p(W)}-1
\right|
\right].
\label{eq:v_infty_def}
\end{align}
Applying the \(L_1\)--\(L_\infty\) H\"older inequality directly to the same
change-of-measure identity used in Theorem~\ref{th-main} gives
\begin{align}
\left|
\overline{\mathrm{gen}}(\mu,P_{W\mid S})
\right|
\le
\sup_{w\in\mathcal W}
\mathbb E_S\left|L_S(w)-L_\mu(w)\right|
\;V_\infty(S;W).
\label{eq:alpha_infty_bound}
\end{align}
For a discrete hypothesis space,
\begin{align}
V_\infty(S;W)
=
\sum_{w\in\mathcal W}
\esssup_{s\sim P_S}
\left|p(w\mid s)-p(w)\right|.
\end{align}
Thus the ``worst-case'' aspect of the \(\alpha=\infty\) endpoint is the
essential-supremum deviation of the algorithm's conditional output law over
datasets. The loss factor itself is only a mean absolute deviation. With only a
first-moment assumption, no universal polynomial decay rate in \(n\) follows
without additional tail or regularity assumptions.
\end{revision}


\section{The Variance Perspective: \texorpdfstring{$\alpha=2$}{alpha=2}}
\label{sec:special_cases}

An important feature of the $V_\alpha$ framework is its ability to interpolate between different types of generalization guarantees.

For $\alpha=2$, the H\"older conjugate is $\alpha'=2$. The $L_2$-norm of the zero-mean random variable $L_S(w) - L_{\mu}(w)$ is, by definition, its standard deviation. This specialization yields a direct and intuitive connection between generalization and the variance of the learning algorithm's output distribution.

\begin{theorem}
	\label{th:variance_bound}
	Suppose the loss $\ell(w,Z)$ is $\sigma^2$-subgaussian under $Z \sim \mu$ for all $w \in \mathcal{W}$. Then the expected generalization error is bounded as:
	\begin{align}
		\left|\overline{\mathrm{gen}}\left(\mu,P_{W|S}\right)\right| \le \int_{\mathcal{W}} \sqrt{\frac{\sigma^2}{n} \mathrm{Var}_S\Big[p(w|S)\Big]} \, dw.
	\end{align}
    \begin{remark}
        For discrete hypothesis spaces, the integral is replaced by a sum.        
    \end{remark}
\end{theorem}
\begin{proof}[Proof Sketch]
    For $\alpha=2$, Corollary~\ref{cor:subgaussian} simplifies, giving $\left|\overline{\mathrm{gen}}\right| \le 2\sqrt{\frac{\sigma^2}{n}} V_2(S;W)$. The theorem follows by applying an identity from \citet[Sec.~III]{mojahedian2019correlation}, which demonstrates that $V_2(S;W) = \int_{\mathcal{W}} \sqrt{\mathrm{Var}_S[p(w|S)]} \, dw$. In Appendix~\ref{proof:th3}, we refine this argument and show that the leading constant $2$ can be removed, yielding the tighter bound stated in the theorem.
\end{proof}

The bound in Theorem~\ref{th:variance_bound} reveals a direct link between generalization and algorithmic stability. The term $\mathrm{Var}_S[p(w|S)]$ quantifies the variability of the algorithm's output probability density across different training samples. High variance implies that the algorithm's decision-making is highly sensitive to the specific realization of the training data—a hallmark of overfitting—which penalizes the bound. Conversely, low variance implies robustness, leading to a tighter generalization guarantee.

\paragraph{Relationship to Existing Stability Notions.} 
We emphasize that this variance-based notion of stability is distinct from classical deterministic measures like uniform stability~\citep{bousquet2002stability}. While uniform stability offers a worst-case guarantee on the change in the loss function from altering a single data point, our measure captures the average-case sensitivity of the entire output probability distribution.

Furthermore, our framework provides a timely complement to recent variance-based perspectives on learnability. For instance, \citet{gastpar2025algorithms} recently established that algorithms exhibit tight generalization bounds when the conditional variance of their \textit{population loss} given the sample, $\mathrm{Var}(L_\mu(A(S))|S)$, is small. While their work focuses on variance in the outcome (loss) space, our $V_2$ measure operates directly on the algorithm's output probability distribution, $\mathrm{Var}_S[p(w|S)]$. Thus, our approach measures algorithmic stability intrinsically within the hypothesis space, offering a parallel, distribution-centric view of generalization.

\subsection{A Geometric Advantage: Reparameterization Invariance}
\label{subsec:reparameterization}

A common criticism of bounds evaluated directly in the hypothesis space $\mathcal{W}$ is that they can be arbitrarily manipulated by changing the parameterization of the model (e.g., scaling the weights of a neural network while adjusting the activations to leave the function unchanged).

\rev{The reparameterization property is not specific to the variance identity at $\alpha=2$: the entire $V_\alpha$ family is invariant under any smooth bijective coordinate change that is independent of the training sample. The $\alpha=2$ case remains especially interpretable because it also admits the density-variance representation in Theorem~\ref{th:variance_bound}.}

\begin{theorem}[Reparameterization Invariance of $V_\alpha$]
    \label{th:reparameterization}
    Let $\phi : \mathcal{W} \to \mathcal{V}$ be a bijective differentiable transformation with differentiable inverse, let $V=\phi(W)$, and assume the relevant densities exist. Then, \rev{for every $\alpha\ge 1$,}
    \begin{align}
        V_\alpha(S;W)=V_\alpha(S;V).
    \end{align}
\end{theorem}
\begin{proof}[Proof Sketch]
    Let $v=\phi(w)$ and write $J_\phi(w)$ for the Jacobian. The conditional and marginal densities transform with the same Jacobian factor:
    \begin{align}
      p_{V\mid S}(v\mid s)&=p_{W\mid S}(w\mid s)|\det J_\phi(w)|^{-1},\\
      p_V(v)&=p_W(w)|\det J_\phi(w)|^{-1}.
    \end{align}
    Hence the likelihood-ratio deviation is unchanged,
    \begin{align}
      \frac{p_{V\mid S}(v\mid s)}{p_V(v)}-1
      =\frac{p_{W\mid S}(w\mid s)}{p_W(w)}-1.
    \end{align}
    Since $P_V$ is the pushforward of $P_W$ under $\phi$, substituting this identity into the definition of $V_\alpha$ gives $V_\alpha(S;V)=V_\alpha(S;W)$. Equivalently, in the density-deviation form, the inverse Jacobian multiplying the density difference cancels the Jacobian in $dv$. The full proof is given in Appendix~\ref{proof:reparameterization}.
\end{proof}

\rev{Thus the dependence term in every member of the $V_\alpha$ family is coordinate-invariant under such reparameterizations. For $\alpha=2$, this general invariance agrees with the direct Jacobian cancellation visible in the density-variance formula.}



\section{Adaptive Density Stability (ADS) \& Practical Rates}
\label{sec:stability}

The variance-based bound in Theorem~\ref{th:variance_bound} provides a powerful conceptual link between stability and generalization. However, a key theoretical question remains: under what formal conditions on a learning algorithm is the variance term $\mathrm{Var}_S[p(w|S)]$ guaranteed to be small enough to yield a non-vacuous, decaying generalization bound? 

To address this, we introduce a novel, strong stability condition. We then show that algorithms satisfying this condition achieve standard $\mathcal{O}(1/n)$ generalization rates within our $V_2$ framework.

\begin{definition}[Adaptive Density Stability]
\label{def:adaptive_stability}
A learning algorithm is said to be $(\gamma_n, n)$-adaptively density stable if for all datasets $S, S'$ of size $n$ that differ in exactly one sample, and for all hypotheses $w \in \mathcal{W}$, the following holds:
\begin{align}
	\left|p(w|S) - p(w|S')\right| \le \gamma_n p(w), \label{eq:adaptive_stability}
\end{align}
where $p(w) = \mathbb{E}_S[p(w|S)]$ is the marginal distribution, and the stability parameter $\gamma_n$ may depend on the sample size $n$.
\end{definition}

This definition formalizes a notion of distributional stability where the allowed change in the output density at a point $w$ is restricted relative to the marginal probability $p(w)$. This adaptively places a stronger stability constraint on the low-probability tails of the output distribution. As shown in Appendix~\ref{app:stability}, this strong, pointwise condition implies classical TV stability and Expected Uniform Stability.

Our main result in this section shows that algorithms satisfying ADS have a strictly controlled $V_2$ measure, yielding a generalization bound.

\begin{theorem}
\label{th:stability_bound}
If a learning algorithm is $(\gamma_n, n)$-adaptively density stable, then its $V_2$ correlation measure is bounded by:
\begin{align}
	V_2(S;W) \le \gamma_n \sqrt{\frac{n}{2}}.
\end{align} 
Consequently, for a $\sigma^2$-subgaussian loss, the generalization error is bounded by:
\begin{align}
	\left|\overline{\mathrm{gen}}\left(\mu,P_{W|S}\right)\right| \le \frac{\sigma \gamma_n}{\sqrt{2}}.
\end{align}
\end{theorem}
\begin{proof}[Proof Sketch]
The proof applies the Efron-Stein inequality to bound the variance of the density ratio, avoiding direct integration over high-probability volumes. The full proof is provided in Appendix~\ref{app:proof_stability}.
\end{proof}

\paragraph{On Non-Vacuousness and Practical Algorithms.} 
At first glance, the bound $\frac{\sigma \gamma_n}{\sqrt{2}}$ appears independent of $n$. The utility of this theorem hinges entirely on whether a learning algorithm possesses a stability parameter $\gamma_n$ that decays with $n$. \rev{As a concrete sufficient example, we verify this property for fixed-temperature Gibbs posteriors under bounded losses.}

\textbf{Example 5.1 (Gibbs Posteriors satisfy ADS).} Consider a learning algorithm that outputs a Gibbs posterior, $p(w|S) \propto \exp(-\beta L_S(w)) \pi(w)$. \rev{This distribution is related to the invariant target of idealized Langevin dynamics under suitable regularity conditions; the ADS calculation below is a statement about the Gibbs posterior itself and does not establish ADS for finite-step or finite-time SGLD iterates.} Suppose the loss function $\ell(w,z)$ is bounded in $[0, M]$.
If we change a single data point $z_i$ to $z_i'$, the empirical risk changes by at most $\Delta L_S = M/n$. Consequently, the ratio of the unnormalized densities is bounded by $\exp(\beta M / n)$. 
This implies the algorithm satisfies $\epsilon$-Differential Privacy (DP) with $\epsilon = 2\beta M / n$, yielding:
\begin{align}
    e^{-\epsilon} \le \frac{p(w|S)}{p(w|S')} \le e^{\epsilon}.
\end{align}

To relate this local guarantee to the marginal distribution \(p(w)\), we invoke a standard group privacy argument: extending the above bound along a path of neighboring datasets yields a uniform bound $p(w|S) \le e^{n\epsilon} p(w)$ for all datasets \(S\). Using this together with the DP inequality, we convert the ratio bound into a difference bound:
\[
|p(w|S) - p(w|S')|
\le (e^\epsilon - 1)\max\{p(w|S), p(w|S')\}
\le (e^\epsilon - 1)e^{2\beta M} p(w).
\]
Thus, the Gibbs posterior is $(\gamma_n, n)$-adaptively density stable with $\gamma_n = \left(e^{\frac{2\beta M}{n}} - 1\right)e^{2\beta M}$. For sufficiently large $n$, a first-order Taylor expansion yields $\gamma_n \approx \frac{2\beta M}{n} e^{2\beta M} = \mathcal{O}(1/n)$. 

\begin{revision}
This $\mathcal O(1/n)$ statement is a fixed-temperature statement. For
$0\le \ell\le M$, Hoeffding's lemma gives $\sigma=M/2$, and
Theorem~\ref{th:stability_bound} yields
\begin{align}
B_{\mathrm{ADS}}
=\frac{M}{2\sqrt{2}}e^{2\beta M}
\left(e^{2\beta M/n}-1\right)
=\frac{\beta M^2 e^{2\beta M}}{\sqrt{2}\,n}+\mathcal O(n^{-2})
\end{align}
when $\beta M$ is fixed. Hence the constant grows rapidly with $\beta M$;
in particular, if $\beta=\Theta(n)$ with fixed $M$, this certificate becomes
exponentially conservative. The fast rate should therefore not be interpreted
as uniform non-vacuousness over inverse temperature.
\end{revision}

\rev{At fixed $\beta M$, substituting this decay into Theorem~\ref{th:stability_bound} gives the corresponding $\mathcal O(1/n)$ generalization rate:}
\begin{align}
	\left|\overline{\mathrm{gen}}\left(\mu, P_{W|S}\right)\right| \le \frac{\sigma}{\sqrt{2}}\cdot \mathcal{O}\left(\frac{1}{n}\right) = \mathcal{O}\left(\frac{1}{n}\right).
\end{align}

\rev{This establishes a concrete pathway from the abstract variance-based $V_2$ bound to finite-sample guarantees for fixed-temperature Gibbs posteriors; extensions to other stochastic algorithms require algorithm-specific ADS arguments.}


\section{Scaling to High Dimensions via Conditional and Projected
\texorpdfstring{$V_\alpha$}{V-alpha} Bounds}
\label{sec:scaling_high_dim}

The bounds in Sections~\ref{sec:main_result} and \ref{sec:special_cases} measure
dependence directly between the training sample \(S\) and the learned hypothesis
\(W\). This parameter-space dependence can be difficult to estimate and can grow
quickly in over-parameterized models. To obtain bounds whose information term is
\emph{sample-size limited} (as in conditional mutual information arguments
\citep{steinke2020}), we move to a supersample formulation and measure how much
the algorithm reveals about the \emph{selector} that chooses which examples are
used for training.

Let
\[
    \tilde S
    =
    \{(\tilde Z_{i,0},\tilde Z_{i,1})\}_{i=1}^n
\]
be a supersample of \(2n\) i.i.d.\ examples from \(\mu\). Let
\[
    U=(U_1,\dots,U_n)\sim \mathrm{Unif}(\{0,1\}^n)
\]
be independent of \(\tilde S\). The selected training sample and ghost sample are
\[
    S_i=\tilde Z_{i,U_i},
    \qquad
    \bar S_i=\tilde Z_{i,1-U_i}.
\]
The learning algorithm is run on \(S\) and outputs \(W\). For fixed \(\tilde S\),
define the selected-versus-ghost gap
\begin{align}
    \hat\Delta_{\tilde S}(w,U)
    \coloneqq
    \frac1n\sum_{i=1}^n
    \left[
        \ell(w,\tilde Z_{i,1-U_i})
        -
        \ell(w,\tilde Z_{i,U_i})
    \right].
    \label{eq:supersample_gap}
\end{align}
The standard supersample identity gives
\[
    \overline{\mathrm{gen}}(\mu,P_{W|S})
    =
    \mathbb E_{\tilde S,U,W}
    \left[
        \hat\Delta_{\tilde S}(W,U)
    \right].
\]

\paragraph{Measure-theoretic convention.}
All density ratios below are interpreted as Radon--Nikodym derivatives with
respect to the relevant marginal law. If the corresponding derivative or
integrability condition fails, the associated \(V_\alpha\) quantity is
interpreted as \(+\infty\).

\paragraph{A single decoupling principle.}
All conditional, projected, sample-wise, and binarized bounds in this section
follow from a centered change-of-measure/Hölder decoupling lemma together with a
conditional data-processing inequality for \(V_\alpha\); both are stated and
proved in Appendix~\ref{app:scalable_bounds_lemmas}.

\subsection{Conditional \texorpdfstring{$V_\alpha$}{V-alpha} bound}
\label{subsec:conditional_valpha_main}

For every fixed \(w\) and \(\tilde S\), the map \(U\mapsto
\hat\Delta_{\tilde S}(w,U)\) is centered under the uniform selector:
\[
    \mathbb E_U\!\left[\hat\Delta_{\tilde S}(w,U)\right]=0.
\]
Applying the master decoupling lemma from Appendix~\ref{app:scalable_bounds_lemmas} with conditioning variable \(C=\tilde S\) yields the following bound.

\begin{theorem}[Conditional \(V_\alpha\) generalization bound]
\label{th:conditional_valpha}
For any learning algorithm and any measurable loss for which the displayed
quantities are well-defined, for every \(\alpha\ge 1\),
\begin{align}
    \left|
        \overline{\mathrm{gen}}(\mu,P_{W|S})
    \right|
    \le
    \mathbb E_{\tilde S}
    \left[
        \esssup_{w\in\mathcal W}
        \left\|
            \hat\Delta_{\tilde S}(w,U)
        \right\|_{\alpha',U}
        \,
        V_\alpha(U;W|\tilde S)
    \right],
    \label{eq:conditional_valpha_bound}
\end{align}
where \(\alpha'\) is the Hölder conjugate of \(\alpha\) and
\begin{align}
    V_\alpha(U;W|\tilde S)
    \coloneqq
    \mathbb E_{W\sim P_{W|\tilde S}}
    \left[
        \left(
            \mathbb E_U
            \left|
                \frac{P_{W|U,\tilde S}}{P_{W|\tilde S}}(W|U,\tilde S)
                -1
            \right|^\alpha
        \right)^{1/\alpha}
    \right].
    \label{eq:conditional_valpha_def}
\end{align}
\end{theorem}

Theorem~\ref{th:conditional_valpha} replaces dependence on the full sample \(S\)
by dependence on the discrete selector \(U\in\{0,1\}^n\). This can mitigate the
dimension-scaling issues of parameter-space quantities, but the dependence term
may still be large if the learned object nearly determines the selector.

\subsection{Projected and estimable variants (overview)}
\label{subsec:projected_consequences}

Estimating \(V_\alpha(U;W|\tilde S)\) can still be difficult when \(W\) is high
dimensional. A natural remedy is to replace \(W\) by a lower-dimensional
statistic that \emph{determines the supersample gap}
\(\hat\Delta_{\tilde S}(W,U)\), e.g.\ supersample predictions, loss differences,
or evaluated losses. A conditional data-processing inequality ensures that such
post-processing cannot increase the \(V_\alpha\) term. We state and prove a
hierarchy of theorem-valid projected, sample-wise, and binarized bounds in
Appendix~\ref{app:scalable_bounds}; Table~\ref{tab:holder_instantiations} summarizes the resulting
options and the statistic used in each case.




\section{Experiments}
\label{sec:experiments}

\rev{We empirically evaluate the proposed sample-wise and binarized \(V_\alpha\) bounds using the BinaryJS experimental pipeline of \citet{dong2025exactly}. Specifically, we retain its binary MNIST \(4\)-vs.-\(9\) CNN and pretrained ResNet-50 CIFAR-10 configurations, and augment the evaluation with the proposed \(V_\alpha\)-based quantities described below.}
For each
\(n\in\{75,250,1000,4000\}\), we draw a supersample
\(\tilde S=\{(\tilde Z_{i,0},\tilde Z_{i,1})\}_{i=1}^n\) and selector
\(U\in\{0,1\}^n\), train a CNN on
\(\{\tilde Z_{i,U_i}\}_{i=1}^n\), and use
\(\{\tilde Z_{i,1-U_i}\}_{i=1}^n\) as the ghost sample. We use \(5\)
independent supersample seeds and \(30\) selector seeds per supersample, train
for \(200\) epochs, and evaluate all quantities at the final checkpoint. Both
\(0\)-\(1\) loss and clipped cross-entropy are evaluated on the full
supersample.

\rev{We focus on conditional, coordinatewise loss-space bounds because the
full-vector quantities are difficult to estimate from only $30$ MNIST (or
$40$ CIFAR-10) selector runs at fixed $\tilde S$. For each coordinate $i$,
we estimate the law of $(U_i,T_i)\mid\tilde S$ separately, evaluate its bound
contribution, and then average over coordinates and supersamples. We compare the loss-difference bound
(Theorem~\ref{th:samplewise_loss_difference_valpha}),
the side-averaged single-loss bound
(Theorem~\ref{th:samplewise_single_loss_valpha}),
the loss-pair bound
(Theorem~\ref{th:samplewise_loss_pair_valpha}),
and their binarized counterparts
(Theorems~\ref{th:samplewise_binarized_single_loss_valpha},
\ref{th:samplewise_binarized_loss_difference_valpha}, and
\ref{th:samplewise_binarized_loss_pair_valpha}),
using the conditional forms in
Appendix~\ref{sec:conditional_samplewise_experimental}.} In the figures and discussion below, T8--T13 denote the corresponding
conditional estimators described above.

\rev{For the fair binary selectors used here, the coordinatewise
quantities are independent of $\alpha$
(Lemma~\ref{lem:fair_binary_valpha}).
We plot the mean, over supersample seeds, of the estimated bound
increment divided by the corresponding absolute empirical
selected-versus-ghost gap. These are descriptive plug-in comparisons
rather than finite-sample certificates.
The Bayesian linear regression (BLR) experiments (Appendix~\ref{app:simulation_blr}) separately
examine $\alpha$ tuning using disjoint Monte Carlo pools for
selection and evaluation.}

\begin{figure}[t]
    \centering
    \includegraphics[width=.88\linewidth]{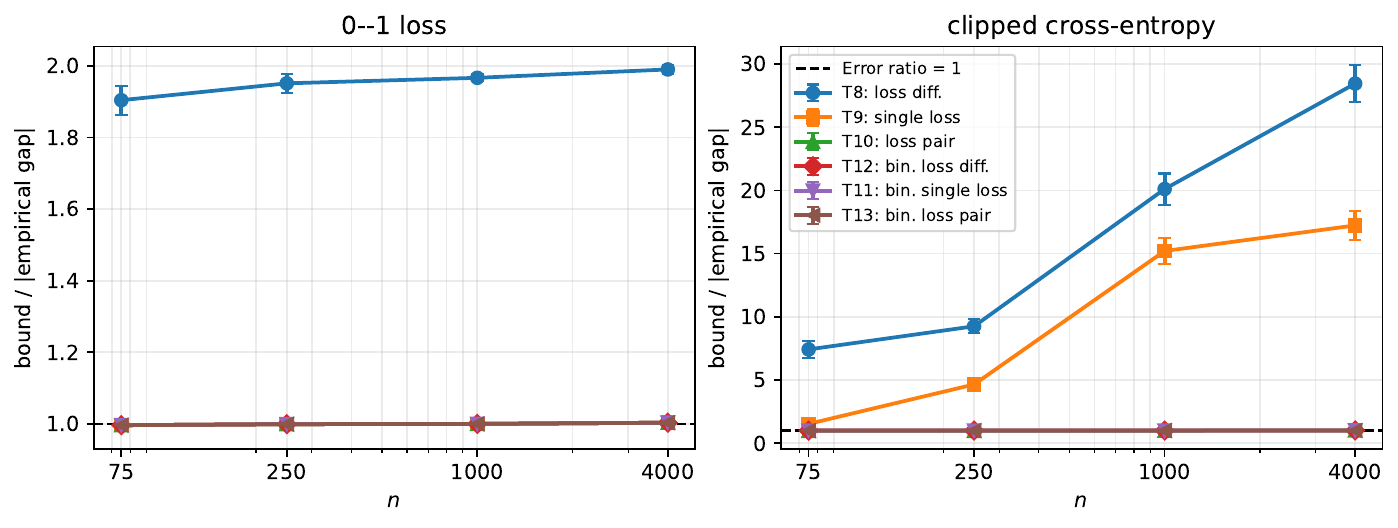}
    \caption{
    Ratio between estimated bound increments and the empirical
    selected-versus-ghost gap on MNIST \(4\) vs.\ \(9\).
    Left: \(0\)-\(1\) loss. Right: clipped cross-entropy. The dashed line
    denotes ratio \(1\). \rev{Error bars show the standard error of the mean
    (SEM) over five supersample seeds.}
    }
    \label{fig:ratio_to_error}
\end{figure}

\rev{Figure~\ref{fig:ratio_to_error} shows that for $0$--$1$ loss,
T9--T13 lie close to the empirical gap, whereas the bounded loss-difference
curve (T8) has ratios of about $1.9$--$2.0$. For clipped cross-entropy,
the loss-pair and binarized estimates (T10--T13) remain close to ratio $1$
across the displayed sample sizes. T8 and T9 are more conservative, with
ratios increasing from about $7.4$ to $28.4$ and $1.5$ to $17.2$,
respectively. T8 uses a worst-case range factor in place of the local
loss-difference magnitude. Thus, these results favor the loss-pair
and binarized statistics for empirical tightness in this setting.}

\rev{For clipped cross-entropy, the same separation between T8/T9 and
T10--T13 is observed under
MNIST label-noise probabilities $0.05$, $0.10$, and $0.15$
(Appendix~\ref{app:additional_deep_experiments}).}

\paragraph{CIFAR-10.}
\begin{revision}
Using the pretrained ResNet-50 configuration from the BinaryJS pipeline
\citep{dong2025exactly}, we evaluate the same conditional, coordinatewise
quantities for $n\in\{1000,5000,20000\}$.
Figure~\ref{fig:cifar10_resnet_ratio} shows that T9--T13 remain close to
ratio $1$ for $0$--$1$ loss, whereas T8 is about $1.8$--$1.9$.
For clipped cross-entropy, T10--T13 remain close to the empirical gap,
while T8 and T9 are more conservative. This reproduces the main qualitative
separation seen on MNIST in a larger model and a ten-class task.
\end{revision}

\begin{figure}[t]
    \centering
    \includegraphics[width=.88\linewidth]{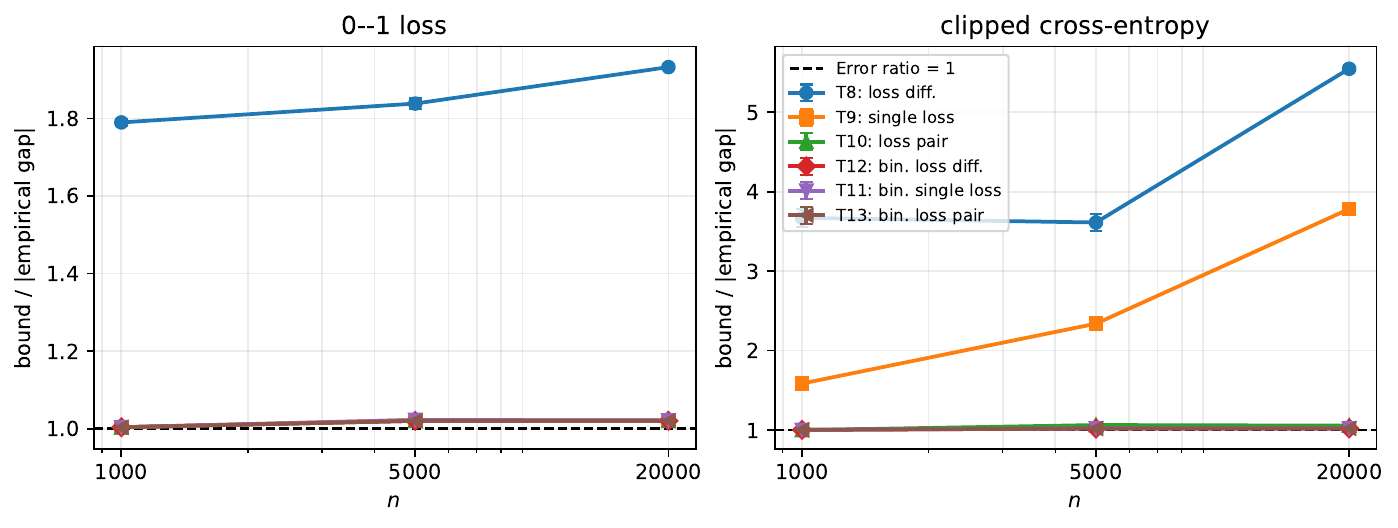}
    \caption{\rev{CIFAR-10 results with a pretrained ResNet-50. The panels show
    the ratio between the estimated bound increment and the empirical
    selected-versus-ghost gap for $0$--$1$ loss (left) and clipped
    cross-entropy (right). The dashed line denotes ratio $1$.
    Error bars show SEM over two supersample seeds.}}
    \label{fig:cifar10_resnet_ratio}
\end{figure}

\paragraph{Bayesian linear regression (BLR).}
To complement the supersample MNIST study (which focuses on estimability of the sample-wise loss-space bounds), we also evaluate the parameter-space and conditional $V_\alpha$ bounds in a conjugate Bayesian linear regression model where $p(w|S)$ is Gaussian and densities can be computed in closed form. The BLR experiments (Appendix~\ref{app:simulation_blr}) illustrate the dimension scaling behavior highlighted in Sections~\ref{sec:main_result}- \ref{sec:scaling_high_dim}: fixed \(\alpha=2\) parameter-space bounds degrade with dimension, while tuning \(\alpha\) can substantially improve tightness, and the conditional supersample formulation is dimension-free but can be numerically conservative in this regime. See Appendix~\ref{app:simulation_blr} (Figures~\ref{fig:blr_extended_conditional}--\ref{fig:blr_extended_conditional_alpha}) for the full setup and results.

\section{Conclusion and Limitations}
\label{sec:conclusion}

We introduced a \(V_\alpha\)-based family of information-theoretic
generalization bounds. The parameter \(\alpha\) separates the loss-sensitivity
term from the dependence term, yielding a tunable interpolation between
TV-like and variance-like regimes. At \(\alpha=2\), the bound admits a
variance-stability interpretation through
\(\mathrm{Var}_S[p(w|S)]\), and this quantity is invariant under smooth
reparameterizations of the hypothesis space. We also introduced adaptive
density stability as a sufficient condition for controlling \(V_2\), and
extended the framework to conditional, projected, sample-wise, and binarized
supersample bounds. Empirically, the finite-state loss-space versions provide
\rev{plug-in estimates that lie close to the observed selected-versus-ghost gap in our MNIST runs. These empirical comparisons are descriptive rather than finite-sample certificates.} \rev{This behavior is also observed under
MNIST label noise and in the CIFAR-10/ResNet-50 experiment.}

\paragraph{Limitations.}
Several limitations remain. First, the base \(V_\alpha(S;W)\) bound is a
parameter-space quantity and can still be large in high-dimensional models.
The conditional and projected bounds avoid direct parameter-space dependence,
but their full-vector versions can be statistically difficult to estimate from
few selector draws. Second, adaptive density stability is a strong pointwise
condition; it is useful for analyzing Gibbs-type randomized algorithms, but it
should not be interpreted as a generic guarantee for deterministic deep
networks without additional smoothing or randomization. Third, \rev{our deep-learning experiments remain limited to a small set of
classification datasets and architectures.} Developing sharper finite-sample estimators for
\(V_\alpha(U;T_{\tilde S}|\tilde S)\) and validating them across larger
architectures and datasets are important directions for future work.

\clearpage
\bibliographystyle{tmlr}
\bibliography{References}
\clearpage

\appendix


\section{Extended Related Work and Modern Baselines}
\label{app:extended_related_work}

The literature on information-theoretic generalization bounds has expanded rapidly since the foundational results connecting MI to expected generalization \citep{russo2019doesdataexplorationoverfit, xu2017information}. While our main text highlights the core narrative, this appendix provides a deeper contextualization of our $V_\alpha$ framework within recent developments in the field.

\subsection{Advancements in Conditional and Sample-Wise Measures}
A fundamental limitation of classical MI bounds is their tendency to become vacuous in high-dimensional or over-parameterized regimes. The introduction of CMI via a supersample framework by \citet{steinke2020} introduced an alternative approach, shifting the dependence from the dimension of the hypothesis space to the size of the training sample.

Building on this, subsequent works have further tightened these guarantees. \citet{hellstrom2022new} introduced a family of generalization bounds using \textit{sample-wise evaluated CMI}, which disaggregates the information measure down to the level of individual data points. Similarly, \citet{wang2023sample} utilized sample-conditioned hypothesis stability, alongside individual mutual information metrics, to provide even sharper, data-dependent bounds. Furthermore, \citet{dong2024towards} extended information-theoretic perspectives to unify generalization bounds beyond standard pointwise learning paradigms.

\textbf{Connection to our work:} The dimensionality challenges identified by these authors apply equally to our base $V_\alpha$ framework (Sections~\ref{sec:main_result} and \ref{sec:special_cases}). By adapting the supersample conditioning technique in Section~\ref{sec:scaling_high_dim}, we extend the $V_\alpha$-based generalization framework to the conditional supersample setting. This conditional extension ensures that the variance-based intuition of $V_2$ can be safely applied to high-capacity models without directly estimating dependence in the full parameter space.

without directly estimating dependence in the full parameter space

\subsection{Divergences, \texorpdfstring{$f$}{f}-Information, and Binary Comparators}
In parallel to conditional frameworks, the community has explored an array of alternative dependency measures to tighten the classical $\sqrt{I(S;W)}$ rate. \citet{harutyunyan2021information} and subsequent works have derived bounds based on Binary KL and Binary Jensen-Shannon (JS) divergences, which often yield tighter empirical estimates for black-box algorithms compared to unbounded continuous measures.

More broadly, \citet{wang2024generalization} recently proposed a highly generalized framework utilizing conditional $f$-information. Because our $V_\alpha$ correlation measure can be viewed as a specific vector-valued $L_p$-norm metric, it is conceptually related to dependence measures based on $f$-divergences.

\textbf{Connection to our work:} While highly abstract frameworks (like $f$-information or Orlicz space unifications~\citep{chu2023unification}) are mathematically encompassing, their generality may make particular algorithmic interpretations less explicit. The distinct advantage of our $V_\alpha$ framework is its tunable parameter ($\alpha$). The explicit H\"older decoupling in Theorem~\ref{th-main} allows us to smoothly transition to the $\alpha=2$ regime, which uniquely collapses into the highly interpretable $\mathrm{Var}_S[p(w|S)]$ expression. Thus, our framework focuses on an explicit family of dependence measures in favor of a clear, actionable link between generalization and algorithmic variance.

\subsection{Variance, Learnability, and Bayesian Stability}
The connection between algorithmic stability and generalization is foundational~\citep{bousquet2002stability, feldman2019high}. Recently, the literature has shifted toward probabilistic and variance-based notions of stability. \citet{moran2023bayesian} cataloged a ``Bayesian stability zoo,'' classifying how different posterior divergences (e.g., KL-stability, TV-stability) govern generalization. 

Most closely related to our $\alpha=2$ result is the recent work by \citet{gastpar2025algorithms}, who demonstrated that an algorithm exhibits tight generalization bounds if and only if the conditional variance of its population loss given the sample is small (i.e., $\mathrm{Var}(L_\mu(A(S)) | S) \le \epsilon$).

\textbf{Connection to our work:} While both our work and \citet{gastpar2025algorithms} isolate \textit{variance} as a fundamental driver of learnability, we operate in complementary spaces. \citet{gastpar2025algorithms} focus on the \textbf{outcome-space variance}—the fluctuation of the final loss metric. In contrast, our $V_2$ bound measures the \textbf{hypothesis-space variance}—the fluctuation of the algorithm's output probability density $\mathrm{Var}_S[p(w|S)]$ over datasets. 

Our approach is agnostic to the loss function during the computation of the stability term (the loss only scales the bound via the subgaussian constant). This makes $V_2$ an intrinsic, geometric property of the learning algorithm itself (as proved by the reparameterization invariance in Theorem~\ref{th:reparameterization}), offering a distinct, model-centric lens on stability.



\section{Proofs for Section~\ref{sec:main_result}: The \texorpdfstring{$V_\alpha$}{V-alpha} Information Framework}
\label{app:proofs_section3}

This section contains the detailed proofs for the general $V_\alpha$ generalization bound, its symmetric property, and the concrete bounds derived under the subgaussian loss assumption.

\subsection{Proof of Theorem~\ref{th-main} (General Generalization Bound)}
\label{proof:th-main}

\begin{proof}
The expected generalization error is defined as the expected difference between the population risk and the empirical risk:
\begin{align}
    \overline{\mathrm{gen}}\left(\mu, P_{W|S}\right) = \mathbb{E}_{P_{W,S}}\left[L_{\mu}(W) - L_{S}(W)\right].
\end{align}
We can rewrite this joint expectation by first conditioning on the hypothesis $W$:
\begin{align}
    \overline{\mathrm{gen}}\left(\mu, P_{W|S}\right) = \mathbb{E}_{W \sim P_W} \left[ \mathbb{E}_{S \sim P_{S|W}} \left[ L_{\mu}(W) - L_{S}(W) \right] \right].
\end{align}
The core of the proof relies on a change-of-measure identity, connecting the conditional expectation over $P_{S|W}$ to an expectation over the marginal data distribution $P_S$:
\begin{align}
    \mathbb{E}_{S \sim P_{S|W}} \left[ L_{\mu}(W) - L_{S}(W) \right] = \mathbb{E}_{S \sim P_S} \left[ \left( L_{\mu}(W) - L_{S}(W) \right) \left( \frac{p_{S|W}(S|W)}{p_S(S)} \right) \right].
    \label{eq:proof_identity}
\end{align}
This change of measure holds because for any measurable function $f(S)$, we have $\mathbb{E}_{S \sim P_{S|W}}[f(S)] = \int f(s) p_{S|W}(s|W) ds = \int f(s) \frac{p_{S|W}(s|W)}{p_S(s)} p_S(s) ds$.

Furthermore, observe that for any fixed hypothesis $W$, the expected empirical risk over freshly drawn data $S \sim P_S$ is simply the population risk:
\begin{align}
    \mathbb{E}_{S \sim \mu^{\otimes n}}\left[L_S(W)\right] &= \mathbb{E}_{S \sim \mu^{\otimes n}}\left[\frac{1}{n}\sum_{i=1}^{n}\ell(W,Z_{i})\right] = \frac{1}{n}\sum_{i=1}^{n}\mathbb{E}_{Z_i \sim \mu}\left[\ell(W,Z_{i})\right] = L_{\mu}(W).
\end{align}
This implies that $\mathbb{E}_{S \sim P_S}[L_\mu(W) - L_S(W)] = 0$. By subtracting this zero-valued term from \eqref{eq:proof_identity}, we can center the density ratio:
\begin{align}
    \mathbb{E}_{S \sim P_{S|W}} \left[ L_{\mu}(W) - L_{S}(W) \right] = \mathbb{E}_{S \sim P_S} \left[ \left( L_{\mu}(W) - L_{S}(W) \right) \left( \frac{p_{S|W}(S|W)}{p_S(S)} - 1 \right) \right].
\end{align}
Substituting this back into the expression for the generalization error and applying the absolute value gives:
\begin{align}
    \left|\overline{\mathrm{gen}}\right| &= \left|\mathbb{E}_{W \sim P_W} \left[ \mathbb{E}_{S \sim P_S} \left[ \left( L_{\mu}(W) - L_{S}(W) \right) \left( \frac{p_{S|W}(S|W)}{p_S(S)} - 1 \right) \right] \right] \right| \\
    &\le \mathbb{E}_{W \sim P_W} \left[ \mathbb{E}_{S \sim P_S} \left[ \left| L_{\mu}(W) - L_{S}(W) \right| \cdot \left| \frac{p_{S|W}(S|W)}{p_S(S)} - 1 \right| \right] \right]. \label{eq:abs_val_step}
\end{align}
We now apply H\"older's inequality to the inner expectation over $S$ with conjugate exponents $\alpha'$ and $\alpha$ (where $1/\alpha + 1/\alpha' = 1$):
\begin{align}
    \left|\overline{\mathrm{gen}}\right| &\le \mathbb{E}_{W \sim P_W} \left[ \left\| L_{\mu}(W) - L_{S}(W) \right\|_{\alpha', P_S} \cdot \left\| \frac{p_{S|W}(S|W)}{p_S(S)} - 1 \right\|_{\alpha, P_S} \right] \\
    &= \mathbb{E}_{W \sim P_W} \left[ \left\| L_{S}(W) - L_{\mu}(W) \right\|_{\alpha'} \cdot \left( \mathbb{E}_{S \sim P_S}\left| \frac{p_{W|S}(W|S) p_S(S)}{p_W(W) p_S(S)} - 1 \right|^\alpha \right)^{1/\alpha} \right], \label{eq:bayes_step}
\end{align}
where in \eqref{eq:bayes_step} we used Bayes' rule, $p_{S|W} = p_{W|S} p_S / p_W$, in the density ratio.

To decouple the data-dependent loss sensitivity from the expectation over $W$, we upper-bound the loss term by taking its supremum over the entire hypothesis space $\mathcal{W}$:
\begin{align}
    \left|\overline{\mathrm{gen}}\right| &\le \sup_{w \in \mathcal{W}} \left\{ \left\| L_S(w) - L_{\mu}(w) \right\|_{\alpha'} \right\} \cdot \mathbb{E}_{W \sim P_W} \left[ \left( \mathbb{E}_{S \sim P_S}\left| \frac{p_{W|S}(W|S)}{p_W(W)} - 1 \right|^\alpha \right)^{1/\alpha} \right] \\
    &= \sup_{w \in \mathcal{W}} \left\{ \left\| L_S(w) - L_{\mu}(w) \right\|_{\alpha'} \right\} \cdot V_{\alpha}(S;W),
\end{align}
where the final step substitutes the formal definition of the correlation measure $V_\alpha(S;W)$. This completes the proof.
\end{proof}

\subsection{Symmetric form of Theorem~\ref{th-main}}
\label{app:symmetry_bound}

\begin{remark}[Symmetry of the bound]
\label{remark:symmetry_app}
The bound of Theorem~\ref{th-main} also admits a symmetric form in terms of
\(V_\alpha(W;S)\):
\begin{align}
\left|\overline{\mathrm{gen}}\left(\mu, P_{W|S}\right)\right|
\le \sup_{s \in \mathcal{Z}^n}
\left\| L_s(W) - L_{\mu}(W) \right\|_{\alpha'} \cdot V_{\alpha}(W;S).
\end{align}
Consequently one may take the minimum of the two forms.
\end{remark}

\begin{proof}
The proof of the alternative, symmetric form of the bound begins from \eqref{eq:abs_val_step} in the proof of Theorem~\ref{th-main}. Instead of resolving the joint expectation by conditioning on $W$ first, we condition on $S$ first:
\begin{align}
    \left|\overline{\mathrm{gen}}\left(\mu, P_{W|S}\right)\right| &\le \mathbb{E}_{P_{W}\otimes P_{S}} \left[ \left| L_{S}(W) - L_{\mu}(W) \right| \cdot \left| \frac{p_{S|W}(S|W)}{p_S(S)} - 1 \right| \right] \\
    &= \mathbb{E}_{S \sim P_S} \left[ \mathbb{E}_{W \sim P_{W}} \left[ \left| L_{S}(W) - L_{\mu}(W) \right| \cdot \left| \frac{p_{S|W}(S|W)}{p_S(S)} - 1 \right| \right] \right].
\end{align}
We then apply H\"older's inequality to the inner expectation over $W$ with conjugate exponents $\alpha'$ and $\alpha$:
\begin{align}
    \left|\overline{\mathrm{gen}}\right| &\le \mathbb{E}_{S \sim P_S} \left[ \left\| L_{S}(W) - L_{\mu}(W) \right\|_{\alpha', P_W} \cdot \left\| \frac{p_{S|W}(S|W)}{p_S(S)} - 1 \right\|_{\alpha, P_W} \right] \\
    &= \mathbb{E}_{S \sim P_S} \left[ \left\| L_{S}(W) - L_{\mu}(W) \right\|_{\alpha'} \cdot \left( \mathbb{E}_{W \sim P_W}\left| \frac{p_{S|W}(S|W)}{p_S(S)} - 1 \right|^\alpha \right)^{1/\alpha} \right].
\end{align}
We pull the supremum of the loss sensitivity term out of the expectation over $S$ (taking the worst-case dataset realization $s \in \mathcal{Z}^n$):
\begin{align}
    \left|\overline{\mathrm{gen}}\right| &\le \sup_{s \in \mathcal{Z}^n} \left\{ \left\| L_s(W) - L_{\mu}(W) \right\|_{\alpha'} \right\} \cdot \mathbb{E}_{S \sim P_S} \left[ \left( \mathbb{E}_{W \sim P_W}\left| \frac{p_{S|W}(S|W)}{p_S(S)} - 1 \right|^\alpha \right)^{1/\alpha} \right] \\
    &= \sup_{s \in \mathcal{Z}^n} \left\{ \left\| L_s(W) - L_{\mu}(W) \right\|_{\alpha'} \right\} \cdot V_{\alpha}(W;S).
\end{align}
Because both Theorem~\ref{th-main} and this derivation provide valid upper bounds, the generalization error is bounded by their minimum.
\end{proof}


\subsection{Explicit subgaussian constant}
\label{app:subgaussian_corollary}

\begin{corollary}[Subgaussian losses]
\label{cor:subgaussian}
Suppose that the loss \(\ell(w,Z)\) is \(\sigma^2\)-subgaussian under \(Z\sim\mu\)
for all \(w\in\mathcal W\). Then for any \(\alpha>1\),
\begin{align}
\left|\overline{\mathrm{gen}}\left(\mu, P_{W|S}\right)\right|
\le
\left(\sqrt{\dfrac{2\sigma^2}{n}}\,
(\alpha')^{1/\alpha'}\,
\Gamma\!\left(\frac{\alpha'}{2}\right)^{1/\alpha'}\right)
V_{\alpha}(S;W).
\label{eq:subgaussian_bound_app}
\end{align}
\end{corollary}

\begin{remark}[Trade-off in \(\alpha\)]
\label{remark:tradeoff}
As \(\alpha\) increases, \(V_\alpha(S;W)\) is non-decreasing
\citep{mojahedian2019correlation}, while \(\alpha'\) decreases, shrinking the
loss-sensitivity prefactor. The best \(\alpha\) is problem-dependent.
\end{remark}

\begin{proof}
The proof relies on combining the general bound from Theorem~\ref{th-main} with a standard concentration result that bounds the $L_p$-norm of a subgaussian random variable. We begin by stating the necessary lemma.

\begin{lemma}[Moments of a Subgaussian Random Variable \citep{Rivasplata2012SubgaussianRV}]
	\label{lemma:subgaussian_moments}
	If a random variable $X$ is zero-mean and $\sigma^2$-subgaussian, then for any $p \ge 1$, its $L_p$-norm is bounded as:
	\begin{align}
		\|X\|_{p} = \left(\mathbb{E}\left|X\right|^p\right)^{1/p} \le \sqrt{2\sigma^2} p^{1/p} \Gamma\left(\frac{p}{2}\right)^{1/p},
	\end{align}
	where $\Gamma(\cdot)$ is the Gamma function.
\end{lemma}

Consider the loss sensitivity term in Theorem~\ref{th-main}: $\sup_{w \in \mathcal{W}} \| L_S(w) - L_{\mu}(w) \|_{\alpha'}$.
The term inside the norm, $L_S(w) - L_{\mu}(w)$, is the difference between the empirical mean of $n$ i.i.d.\ random variables and its true population expectation. 

By the assumption of the Corollary, the loss evaluated on a single data point, $\ell(w,Z)$, is $\sigma^2$-subgaussian for any fixed $w \in \mathcal{W}$. A standard property of independent subgaussian random variables is that their sum is also subgaussian, and the variance parameter of the average scales inversely with $n$. Therefore, the empirical risk $L_S(w) = \frac{1}{n} \sum_{i=1}^n \ell(w,Z_i)$ is a $\frac{\sigma^2}{n}$-subgaussian random variable.

This implies that the zero-mean random variable $X_w = L_S(w) - L_{\mu}(w)$ is $\frac{\sigma^2}{n}$-subgaussian. We can now apply Lemma~\ref{lemma:subgaussian_moments} with $p=\alpha'$ and the adjusted variance parameter $\frac{\sigma^2}{n}$ to bound its $L_{\alpha'}$-norm:
\begin{align}
    \left\| L_S(w) - L_{\mu}(w) \right\|_{\alpha'} \le \sqrt{2 \left( \frac{\sigma^2}{n} \right)} (\alpha')^{1/\alpha'} \Gamma\left(\frac{\alpha'}{2}\right)^{1/\alpha'}.
\end{align}
Crucially, because the $\sigma^2$-subgaussian assumption on the loss holds uniformly for all $w \in \mathcal{W}$, this bound also holds strictly for the supremum over $\mathcal{W}$.

Finally, substituting this uniform inequality back into the main bound from Theorem~\ref{th-main} yields the desired result:
\begin{align}
	\left|\overline{\mathrm{gen}}\left(\mu,P_{W|S}\right)\right| &\le \sup_{w\in \mathcal{W}} \left\{\| L_S(w) - L_{\mu}(w) \|_{\alpha'}\right\} \cdot V_{\alpha}(S;W) \\
	&\le \left(\sqrt{\dfrac{2\sigma^2}{n}} (\alpha')^{1/\alpha'} \Gamma\left(\frac{\alpha'}{2}\right)^{1/\alpha'}\right) V_{\alpha}(S;W).
\end{align}
This completes the proof.
\end{proof}

\begin{revision}
\subsection{Proof of Corollary~\ref{cor:heavy_tailed}}
\label{app:heavy_tail_proof}

\begin{proof}
Fix \(w\in\mathcal W\) and define
\begin{align}
X_i(w):=\ell(w,Z_i)-L_\mu(w),
\qquad i=1,\ldots,n.
\end{align}
The variables \(X_1(w),\ldots,X_n(w)\) are independent, have mean zero, and
have finite \(q\)-th moments under \eqref{eq:heavy_tail_mq}. Moreover,
\begin{align}
L_S(w)-L_\mu(w)
=
\frac1n\sum_{i=1}^n X_i(w).
\end{align}
For independent mean-zero random variables and \(1\le q\le2\), the
von Bahr--Esseen inequality \citep{vonbahr1965inequalities} gives the convenient
bound
\begin{align}
\mathbb E\left|\sum_{i=1}^n X_i(w)\right|^q
\le
2\sum_{i=1}^n \mathbb E|X_i(w)|^q.
\label{eq:vbe_heavy_tail}
\end{align}
Therefore,
\begin{align}
\mathbb E_S\left|L_S(w)-L_\mu(w)\right|^q
&=
\frac1{n^q}
\mathbb E\left|\sum_{i=1}^n X_i(w)\right|^q \\
&\le
2n^{1-q}
\mathbb E_Z\left|\ell(w,Z)-L_\mu(w)\right|^q.
\end{align}
Taking the \(q\)-th root gives
\begin{align}
\left\|L_S(w)-L_\mu(w)\right\|_q
\le
2^{1/q}n^{1/q-1}
\left\|\ell(w,Z)-L_\mu(w)\right\|_q.
\end{align}
Since \(q=\alpha/(\alpha-1)\), we have
\(1/q-1=-1/\alpha\). Taking the supremum over \(w\), using the definition of
\(M_q\), and applying Theorem~\ref{th-main} yields
\begin{align}
\left|
\overline{\mathrm{gen}}(\mu,P_{W\mid S})
\right|
\le
2^{1/q}n^{-1/\alpha}M_qV_\alpha(S;W),
\end{align}
which proves \eqref{eq:heavy_tail_bound}.
\end{proof}
\end{revision}


\section{Proofs for Section~\ref{sec:special_cases}: The Variance Perspective (\texorpdfstring{$\alpha=2$}{alpha=2})}
\label{app:proofs_section4}

This section details the proofs for the specific case of $\alpha=2$, which yields the variance-based generalization bound, and establishes the strict geometric invariance of this bound to reparameterizations of the hypothesis space.

\subsection{Proof of Theorem~\ref{th:variance_bound} (The Variance-Based Bound)}
\label{proof:th3}

\begin{proof}
The variance-based bound is a direct and insightful specialization of our main result (Theorem~\ref{th-main}) coupled with the subgaussian loss properties (Corollary~\ref{cor:subgaussian}). The proof proceeds in three distinct steps.

\paragraph{Step 1: Specialize the general bound to $\alpha=2$.}
We begin with the general bound from Theorem~\ref{th-main}:
\begin{align}
    \left|\overline{\mathrm{gen}}\left(\mu, P_{W|S}\right)\right| &\le \sup_{w \in \mathcal{W}} \left\| L_S(w) - L_{\mu}(w) \right\|_{\alpha'} \cdot V_{\alpha}(S;W).
\end{align}
For $\alpha=2$, the H\"older conjugate is $\alpha'=2$. The $L_2$-norm of the zero-mean random variable $L_S(w) - L_{\mu}(w)$ over the dataset $S$ is, by definition, its standard deviation:
\begin{align}
    \left\| L_S(w) - L_{\mu}(w) \right\|_{2} &= \left( \mathbb{E}_S\left[ (L_S(w) - L_\mu(w))^2 \right] \right)^{1/2} \\
    &= \sqrt{\mathrm{Var}_S[L_S(w)]}.
\end{align}
Substituting $\alpha=2$ and this identity into the general bound gives an intermediate formulation:
\begin{align}
    \left|\overline{\mathrm{gen}}\left(\mu, P_{W|S}\right)\right| \le \sup_{w \in \mathcal{W}} \left\{ \sqrt{\mathrm{Var}_S[L_S(w)]} \right\} \cdot V_2(S;W). \label{eq:proof_var_interim}
\end{align}

\paragraph{Step 2: Apply the subgaussian assumption.}
The theorem assumes that the loss $\ell(w,Z)$ evaluated on a single sample is $\sigma^2$-subgaussian for all $w$. As detailed in the proof of Corollary~\ref{cor:subgaussian} (Appendix~\ref{app:subgaussian_corollary}), the independence of the $n$ samples means the empirical risk $L_S(w)$ is a $\frac{\sigma^2}{n}$-subgaussian random variable. A defining property of subgaussian variables (see Definition~\ref{def:subgaussian}) is that their variance is strictly bounded by their subgaussian parameter. Therefore, we have:
\begin{align}
    \mathrm{Var}_S[L_S(w)] \le \frac{\sigma^2}{n} \quad \text{for all } w \in \mathcal{W}.
\end{align}
Because this bound holds uniformly across the hypothesis space, we can substitute it into the supremum in \eqref{eq:proof_var_interim} to obtain:
\begin{align}
    \left|\overline{\mathrm{gen}}\left(\mu, P_{W|S}\right)\right| \le \sqrt{\frac{\sigma^2}{n}} \cdot V_2(S;W). \label{eq:proof_v2_bound}
\end{align}

\paragraph{Step 3: Express $V_2(S;W)$ in terms of variance.}
The final step translates the abstract correlation measure into a tangible probability-space variance. We utilize an analytical identity for the $V_2$ measure derived by \citet[Sec.~III]{mojahedian2019correlation}. For a continuous hypothesis space $\mathcal{W}$ with probability density $p(w|S)$, the $V_2$ correlation measure can be rewritten as the integral of the standard deviation of the density:
\begin{align}
    V_2(S;W) = \int_{\mathcal{W}} \sqrt{\mathrm{Var}_S\Big[p(w|S)\Big]} \, dw. \label{eq:proof_v2_identity}
\end{align}
Substituting this identity into \eqref{eq:proof_v2_bound} and bringing the constant $\sqrt{\sigma^2/n}$ inside the integral yields the final result:
\begin{align}
    \left|\overline{\mathrm{gen}}\left(\mu, P_{W|S}\right)\right| \le \int_{\mathcal{W}} \sqrt{\frac{\sigma^2}{n} \mathrm{Var}_S\Big[p(w|S)\Big]} \, dw.
\end{align}
(Note: For a discrete hypothesis space, the integral naturally reverts to a sum over $w \in \mathcal{W}$). This completes the proof.
\end{proof}

\subsection{Proof of Theorem~\ref{th:reparameterization} (Reparameterization Invariance)}
\label{proof:reparameterization}

\begin{proof}
Let $\phi:\mathcal W\to\mathcal V$ be a bijective differentiable mapping with differentiable inverse, let $V=\phi(W)$, and let $J_\phi(w)$ denote its Jacobian. For $v=\phi(w)$, the change-of-variables formula gives
\begin{align}
 p_{V\mid S}(v\mid s)
 &=p_{W\mid S}(w\mid s)|\det J_\phi(w)|^{-1},\\
 p_V(v)
 &=p_W(w)|\det J_\phi(w)|^{-1}.
\end{align}
The Jacobian depends on the coordinate map and $w$, but not on the random realization of $S$. Therefore
\begin{align}
\frac{p_{V\mid S}(v\mid s)}{p_V(v)}-1
=
\frac{p_{W\mid S}(w\mid s)}{p_W(w)}-1.
\label{eq:valpha_ratio_reparam}
\end{align}
For any $\alpha\ge1$, define
\begin{align}
 g_W(w)
 =\left(
 \mathbb E_{S\sim P_S}
 \left|
 \frac{p_{W\mid S}(w\mid S)}{p_W(w)}-1
 \right|^\alpha
 \right)^{1/\alpha},
\end{align}
and define $g_V(v)$ analogously. \Eqref{eq:valpha_ratio_reparam} implies $g_V(\phi(w))=g_W(w)$. Since $P_V$ is the pushforward of $P_W$ under $\phi$,
\begin{align}
V_\alpha(S;V)
&=\mathbb E_{V\sim P_V}[g_V(V)]\\
&=\mathbb E_{W\sim P_W}[g_V(\phi(W))]\\
&=\mathbb E_{W\sim P_W}[g_W(W)]\\
&=V_\alpha(S;W).
\end{align}
Thus the reparameterization invariance holds for the full $V_\alpha$ family. At $\alpha=2$, the same result can alternatively be read directly from the density-variance identity: the inverse-Jacobian factor in the standard deviation of the transformed density cancels the Jacobian in the volume element.
\end{proof}
\section{Deeper Analysis of Adaptive Density Stability}
\label{app:stability}

In this section, we provide a more detailed analysis of the Adaptive Density Stability condition introduced in Definition~\ref{def:adaptive_stability}. We demonstrate that despite its novelty, our definition is well-grounded within the broader literature on algorithmic stability by formally connecting it to several established concepts. These connections show that our condition is a strong and meaningful notion of distributional stability.

\subsection{Connection to Total Variation (TV) Stability}
Our stability notion is a strong, pointwise condition that directly implies the standard integrated notion of TV stability. An algorithm is considered TV-stable if the TV distance between its output distributions on any two neighboring datasets $S$ and $S'$ is bounded. We show that our condition provides such a bound.

By integrating our stability condition from Definition~\ref{def:adaptive_stability} over the hypothesis space $\mathcal{W}$, we can bound the TV distance:
\begin{align}
    \|p(w|S)-p(w|S')\|_\mathrm{TV} &= \frac{1}{2} \int_{\mathcal{W}} \left|p(w|S) - p(w|S')\right| \, dw \\
    &\leq \frac{1}{2} \int_{\mathcal{W}} \gamma_n p(w) \, dw \quad \text{(by Definition~\ref{def:adaptive_stability})} \\
    &= \frac{\gamma_n}{2} \int_{\mathcal{W}} p(w) \, dw = \frac{\gamma_n}{2},
\end{align}
since $p(w)$ is a probability density. This proves that an algorithm that is $(\gamma_n, n)$-adaptively density stable is also $\frac{\gamma_n}{2}$-TV stable. Our pointwise criterion is therefore strictly stronger than the integrated TV stability condition.

\subsection{Connection to Expected Uniform Stability}
Furthermore, our distributional stability implies the classical notion of stability for the expected loss.
Expected Uniform Stability for randomized algorithms extends the classical notion of uniform stability introduced by \citet{bousquet2002stability} to stochastic learners by taking expectation over the algorithm’s randomness, as formalized in \citet{elisseeff2005stability} and widely used in later works such as \citet{hardt2016train}.
Assuming a loss function $\ell(w, z)$ is bounded by a constant $M$, an algorithm is $\beta$-stable if:
\begin{equation}
    \sup_{z \in \mathcal{Z}} \left| \mathbb{E}_{W \sim p_{W|S}}[\ell(W, z)] - \mathbb{E}_{W' \sim p_{W|S'}}[\ell(W', z)] \right| \leq \beta.
\end{equation}
\begin{revision}
For a loss satisfying $0\le \ell\le M$, the TV implication above gives the
sharper quantitative relation
\begin{align}
\sup_{z\in\mathcal Z}
\left| \mathbb{E}_{W \sim p_{W|S}}[\ell(W, z)]
      -\mathbb{E}_{W' \sim p_{W|S'}}[\ell(W', z)] \right|
&\le M\,\|P_{W\mid S}-P_{W\mid S'}\|_{\mathrm{TV}} \\
&\le \frac{M\gamma_n}{2}.
\end{align}
Thus ADS implies Expected Uniform Stability with parameter $M\gamma_n/2$ for
bounded nonnegative losses. The converse need not hold: Expected Uniform
Stability only detects changes in the output distribution through the loss,
whereas ADS controls the full output law pointwise.
\end{revision}

\subsection{Comparison with Differential Privacy}
It is also instructive to compare our stability notion to DP, another powerful framework for ensuring distributional stability. An algorithm satisfying pure $\epsilon$-DP adheres to:
$$ p(w|S) \leq e^{\epsilon} p(w|S') \quad \forall w, S, S'. $$
This multiplicative, worst-case guarantee is different from our additive, marginal-relative guarantee. However, both enforce strong constraints on the output distribution. A known property of DP (often called group privacy) allows it to be extended to a bound relative to the marginal $p(w)$, yielding $\frac{p(w|S)}{p(w)} \leq e^{n\epsilon}$. This has a direct consequence on the term inside our $V_2$ measure, implying an upper bound on $V_2$ itself: $V_2(S;W) \leq e^{n\epsilon} - 1 \approx n\epsilon$ for small $\epsilon$. This suggests a deep connection between the stability required for privacy and the stability required for generalization in our framework.

\subsection{Comparison with Bayesian Stability}
The framework of Bayesian stability, as surveyed in \citet{moran2023bayesian}, typically analyzes stability by measuring the dissimilarity $d(A(S), \mathcal{P})$ between the posterior $A(S)$ and a reference distribution $\mathcal{P}$ (often a prior), for example using the KL divergence. Our approach offers a complementary perspective. It is a \textbf{pairwise, leave-one-out style constraint}, akin to Differential Privacy, rather than a direct comparison to a fixed reference distribution.

The primary distinction is that our notion is a \textbf{pointwise} constraint on the absolute difference of the densities, whereas a measure like KL-stability is an \textbf{integrated} constraint. This makes our condition formally stronger in certain respects; it guarantees that for \textit{every single hypothesis} $w$, the change in its probability density is controlled, forbidding large local deviations that an integrated measure like KL divergence might permit. This pointwise structure is also tailored for direct use with the variance-based bounds in our paper, as the term $|p(w|S) - p(w|S')|$ appears naturally in the Efron-Stein inequality analysis. In essence, our work proposes a stability definition that is not strictly ``Bayesian'' in the sense of~\citet{moran2023bayesian} but is a strong, distributional, and well-suited tool for our specific analytical pathway.


\section{Proofs for Section~\ref{sec:stability}: Adaptive Density Stability \& Practical Rates}
\label{app:proofs_section5}

This section contains the formal proofs demonstrating how ADS yields strictly bounded $V_2$ measures, and provides the rigorous step-by-step derivation showing that Gibbs posteriors naturally satisfy this condition.

\subsection{Proof of Theorem~\ref{th:stability_bound} (Generalization via ADS)}
\label{app:proof_stability}

\begin{proof}
We are tasked with bounding the $V_2(S;W)$ correlation measure under the assumption that the learning algorithm satisfies $(\gamma_n, n)$-adaptive density stability. Recall from Definition~\ref{def:adaptive_stability} that for any two datasets $S, S'$ differing in exactly one sample, we have:
\begin{align}
    \left| p(w|S) - p(w|S') \right| \le \gamma_n p(w),
\end{align}
where $p(w) = \mathbb{E}_S[p(w|S)]$.

To bound the variance of the output density over the randomness of the dataset $S$, we apply the Efron-Stein inequality~\citep{boucheron2013concentration}. Let $S^{(i)}$ denote the dataset $S$ where the $i$-th observation $Z_i$ is replaced by an independent copy $Z_i'$ drawn from $\mu$. The Efron-Stein inequality bounds the variance of a function of independent random variables by the expected sum of squared differences upon resampling:
\begin{align}
    \mathrm{Var}_S\Big[ p(w|S) \Big] \le \frac{1}{2} \sum_{i=1}^n \mathbb{E}_{S, Z_i'}\left[ \left( p(w|S) - p(w|S^{(i)}) \right)^2 \right].
\end{align}
Because $S$ and $S^{(i)}$ are neighboring datasets (differing only in the $i$-th index), our ADS condition guarantees that $|p(w|S) - p(w|S^{(i)})| \le \gamma_n p(w)$. Substituting this pointwise bound into the inequality yields:
\begin{align}
    \mathrm{Var}_S\Big[ p(w|S) \Big] &\le \frac{1}{2} \sum_{i=1}^n \mathbb{E}_{S, Z_i'}\left[ \left( \gamma_n p(w) \right)^2 \right] \\
    &= \frac{n}{2} \gamma_n^2 p(w)^2.
\end{align}
Taking the square root of both sides, we isolate the standard deviation:
\begin{align}
    \sqrt{\mathrm{Var}_S\Big[ p(w|S) \Big]} \le \gamma_n p(w) \sqrt{\frac{n}{2}}.
\end{align}
We now substitute this bound into the continuous identity for $V_2(S;W)$ (from Equation~\ref{eq:proof_v2_identity} in Appendix~\ref{proof:th3}):
\begin{align}
    V_2(S;W) &= \int_{\mathcal{W}} \sqrt{\mathrm{Var}_S\Big[p(w|S)\Big]} \, dw \\
    &\le \int_{\mathcal{W}} \gamma_n p(w) \sqrt{\frac{n}{2}} \, dw \\
    &= \gamma_n \sqrt{\frac{n}{2}} \int_{\mathcal{W}} p(w) \, dw \\
    &= \gamma_n \sqrt{\frac{n}{2}}.
\end{align}

Finally, applying the variance-based generalization bound from Theorem~\ref{th:variance_bound}:
\begin{align}
    \left|\overline{\mathrm{gen}}\left(\mu, P_{W|S}\right)\right| \le \sqrt{\frac{\sigma^2}{n}} \cdot V_2(S;W) \le \sqrt{\frac{\sigma^2}{n}} \left( \gamma_n \sqrt{\frac{n}{2}} \right) = \frac{\sigma \gamma_n}{\sqrt{2}}. \label{eq:fast_rate_derivation}
\end{align}
This completes the proof. 

\textbf{Remark (Rate Implications).} 
The derivation in \eqref{eq:fast_rate_derivation} shows that the generalization error scales linearly with the stability parameter:
\begin{align}
    \left|\overline{\mathrm{gen}}\left(\mu, P_{W|S}\right)\right| = \mathcal{O}(\gamma_n).
\end{align}
Consequently, if an algorithm satisfies $\gamma_n = \mathcal{O}(1/n)$, our framework yields a generalization bound of order $\mathcal{O}(1/n)$.
\end{proof}

\subsection{Rigorous Proof of Example 5.1 (Gibbs Posteriors satisfy ADS)}
\label{app:proof_gibbs}

\begin{proof}
We formally demonstrate that a Gibbs posterior algorithm satisfies the Adaptive Density Stability condition. Consider a learning algorithm that outputs a probability distribution over hypotheses proportional to the exponentiated negative empirical risk:
\begin{align}
    p(w|S) = \frac{1}{Z(S)} \exp\left(-\beta L_S(w)\right) \pi(w),
\end{align}
where $L_S(w) = \frac{1}{n} \sum_{i=1}^n \ell(w, Z_i)$ is the empirical risk, $\pi(w)$ is a prior distribution, $\beta > 0$ is an inverse temperature parameter, and $Z(S) = \int \exp(-\beta L_S(w)) \pi(w) dw$ is the normalizing partition function. 

Assume the loss function is strictly bounded such that $0 \le \ell(w, z) \le M$ for all $w \in \mathcal{W}$ and $z \in \mathcal{Z}$. Let $S$ and $S'$ be any two neighboring datasets that differ in exactly one observation (e.g., the $k$-th data point is $Z_k$ in $S$ and $Z_k'$ in $S'$). 

The maximum change in the empirical risk for any hypothesis $w$ is bounded by:
\begin{align}
    \left| L_S(w) - L_{S'}(w) \right| = \frac{1}{n} \left| \ell(w, Z_k) - \ell(w, Z_k') \right| \le \frac{M}{n}.
\end{align}
We first bound the ratio of the unnormalized densities:
\begin{align}
    \exp\left( -\beta \frac{M}{n} \right) \le \frac{\exp(-\beta L_S(w)) \pi(w)}{\exp(-\beta L_{S'}(w)) \pi(w)} \le \exp\left( \beta \frac{M}{n} \right).
\end{align}
Because this inequality holds for all $w$, integrating over the hypothesis space $\mathcal{W}$ guarantees that the partition functions are similarly bounded:
\begin{align}
    \exp\left( -\beta \frac{M}{n} \right) \le \frac{Z(S)}{Z(S')} \le \exp\left( \beta \frac{M}{n} \right).
\end{align}
Combining these bounds allows us to bound the ratio of the normalized posterior densities:
\begin{align}
    \frac{p(w|S)}{p(w|S')} &= \left( \frac{\exp(-\beta L_S(w))}{\exp(-\beta L_{S'}(w))} \right) \left( \frac{Z(S')}{Z(S)} \right) \\
    &\le \exp\left( \beta \frac{M}{n} \right) \exp\left( \beta \frac{M}{n} \right) = \exp\left( \frac{2\beta M}{n} \right).
\end{align}
By symmetry, the lower bound is $\exp(-2\beta M/n)$. Let $\epsilon = \frac{2\beta M}{n}$. We have established that the algorithm satisfies pure $\epsilon$-Differential Privacy:
\begin{align}
    e^{-\epsilon} \le \frac{p(w|S)}{p(w|S')} \le e^{\epsilon}, \label{eq:gibbs_dp}
\end{align}
for every pair of neighboring datasets $S,S'$.

Next, we extend this local bound to arbitrary pairs of datasets using the standard group privacy (composition) argument. Let $S_1,S_2 \in \mathcal{Z}^n$ be arbitrary datasets. One can connect $S_1$ to $S_2$ by a sequence of at most $n$ neighboring datasets, each differing in one coordinate. Applying \eqref{eq:gibbs_dp} successively along this path gives:
\begin{align}
    \frac{p(w|S_1)}{p(w|S_2)}
    \le e^{n\epsilon}
    = e^{2\beta M},
\end{align}
and similarly,
\begin{align}
    \frac{p(w|S_1)}{p(w|S_2)}
    \ge e^{-n\epsilon}
    = e^{-2\beta M}.
\end{align}
Hence, for every fixed dataset $S$ and every $w$,
\begin{align}
    p(w|S) \le e^{2\beta M} p(w|S_2)
\end{align}
for all datasets $S_2$. Taking expectation over $S_2 \sim \mu^{\otimes n}$ yields:
\begin{align}
    p(w|S)
    \le
    e^{2\beta M}\,\mathbb{E}_{S_2}\big[p(w|S_2)\big]
    =
    e^{2\beta M} p(w).
\end{align}
The same argument also implies:
\begin{align}
    p(w|S') \le e^{2\beta M} p(w).
\end{align}
Therefore,
\begin{align}
    \max\{p(w|S),p(w|S')\} \le e^{2\beta M} p(w).
\end{align}

Now return to the original neighboring pair $S,S'$. From \eqref{eq:gibbs_dp}, if $p(w|S)\ge p(w|S')$, then:
\begin{align}
    p(w|S)-p(w|S')
    \le
    (e^\epsilon-1)p(w|S').
\end{align}
If instead $p(w|S')\ge p(w|S)$, then:
\begin{align}
    p(w|S')-p(w|S)
    \le
    (e^\epsilon-1)p(w|S).
\end{align}
In either case, we obtain:
\begin{align}
    |p(w|S)-p(w|S')|
    \le
    (e^\epsilon-1)\max\{p(w|S),p(w|S')\}.
\end{align}
Using the marginal bound above gives:
\begin{align}
    |p(w|S)-p(w|S')|
    \le
    (e^\epsilon-1)e^{2\beta M}p(w).
\end{align}
Recalling that $\epsilon = 2\beta M/n$, this becomes:
\begin{align}
    |p(w|S)-p(w|S')|
    \le
    \left(\exp\left(\frac{2\beta M}{n}\right)-1\right)e^{2\beta M}p(w).
\end{align}
Thus, the Gibbs posterior algorithm is $(\gamma_n,n)$-adaptively density stable with parameter:
\begin{align}
    \gamma_n
    =
    \left(\exp\left(\frac{2\beta M}{n}\right)-1\right)e^{2\beta M}.
\end{align}

Finally, since the factor $e^{2\beta M}$ is a constant independent of $n$, a first-order Taylor expansion of the exponential function implies
\begin{align}
    \gamma_n
    =
    e^{2\beta M}\left(\exp\left(\frac{2\beta M}{n}\right)-1\right)
    \approx
    e^{2\beta M}\frac{2\beta M}{n}
    =
    \mathcal{O}\left(\frac{1}{n}\right)
    \qquad \text{as } n\to\infty.
\end{align}
This establishes that Gibbs posteriors satisfy Adaptive Density Stability with the desired $\mathcal{O}(1/n)$ decay rate.
\end{proof}


\section{Master decoupling lemmas for the scalable bounds}
\label{app:scalable_bounds_lemmas}

As in Section~\ref{sec:scaling_high_dim}, all displayed density ratios are
interpreted as Radon--Nikodym derivatives with respect to the relevant marginal
law. If the corresponding absolute-continuity or integrability condition fails,
the associated \(V_\alpha\) quantity is interpreted as \(+\infty\).

\subsection{A master \texorpdfstring{$V_\alpha$}{V-alpha}--Hölder decoupling lemma}
\label{subsec:master_holder}

The following lemma is the single analytic step behind all conditional,
output-space, loss-space, sample-wise, and binarized bounds.

\begin{lemma}[Centered \(V_\alpha\)--Hölder decoupling]
\label{lem:master_holder}
Let \((A,T,C)\) be random variables on standard Borel spaces. For
\(P_C\)-almost every \(c\), let \(P_{A|C=c}\) and \(P_{T|C=c}\) denote the
conditional marginal laws. Let \(\phi_c(t,a)\) be measurable and suppose that,
for \(P_C\)-almost every \(c\),
\begin{align}
    \mathbb E_{A\sim P_{A|C=c}}
    \left[
        \phi_c(t,A)
    \right]
    =
    0
    \qquad
    \text{for every } t .
    \label{eq:master_centering}
\end{align}
For \(\alpha\ge 1\), let \(\alpha'\) be the Hölder conjugate, with
\(\alpha'=\infty\) when \(\alpha=1\). Then
\begin{align}
    \left|
        \mathbb E[\phi_C(T,A)]
    \right|
    \le
    \mathbb E_{C,T}
    \left[
        \left\|
            \phi_C(T,\cdot)
        \right\|_{L_{\alpha'}(P_{A|C})}
        \left(
            \mathbb E_{A\sim P_{A|C}}
            \left|
                \frac{P_{T|A,C}}{P_{T|C}}(T|A,C)
                -1
            \right|^\alpha
        \right)^{1/\alpha}
    \right].
    \label{eq:master_holder_weighted}
\end{align}
Consequently,
\begin{align}
    \left|
        \mathbb E[\phi_C(T,A)]
    \right|
    \le
    \mathbb E_C
    \left[
        \esssup_t
        \left\|
            \phi_C(t,\cdot)
        \right\|_{L_{\alpha'}(P_{A|C})}
        \,
        V_\alpha(A;T|C)
    \right],
    \label{eq:master_holder_sup}
\end{align}
where
\begin{align}
    V_\alpha(A;T|C=c)
    \coloneqq
    \mathbb E_{T\sim P_{T|C=c}}
    \left[
        \left(
            \mathbb E_{A\sim P_{A|C=c}}
            \left|
                \frac{P_{T|A,C}}{P_{T|C}}(T|A,C=c)
                -1
            \right|^\alpha
        \right)^{1/\alpha}
    \right].
    \label{eq:conditional_valpha_general_def}
\end{align}
\end{lemma}

\begin{revision}
\begin{proof}
Fix $c$ for which the stated conditions hold, and write
$r_c(t,a)=\frac{dP_{A,T\mid c}}{d(P_{A\mid c}\otimes P_{T\mid c})}(a,t)$.
By centering and change of measure,
\[
 \mathbb E[\phi_c(T,A)\mid C=c]
 =\int P_{T\mid c}(dt)\int P_{A\mid c}(da)\,
       \phi_c(t,a)\bigl(r_c(t,a)-1\bigr).
\]
Taking absolute values and applying H\"older's inequality to the inner
integral, then averaging over $C$, gives
\eqref{eq:master_holder_weighted}. Bounding the loss norm by its essential
supremum in $t$ gives \eqref{eq:master_holder_sup}. Cases with a divergent
right-hand side follow the convention stated above.
\end{proof}
\end{revision}

\begin{lemma}[Conditional data processing for \(V_\alpha\)]
\label{lem:valpha_dpi}
If \(A\to T\to R\) is a Markov chain conditionally on \(C\), then for every
\(\alpha\ge 1\),
\begin{align}
    V_\alpha(A;R|C)
    \le
    V_\alpha(A;T|C)
    \qquad
    (P_C\text{-a.s.}).
    \label{eq:valpha_dpi}
\end{align}
\end{lemma}

\begin{proof}[Proof sketch]
Fix \(c\) outside a \(P_C\)-null set. Under the conditional law \(P_{\cdot\,|\,C=c}\),
the Markov property implies that \(R\) is obtained from \(T\) by a stochastic
kernel acting only on the \(T\)-coordinate (keeping \(A\) unchanged). Therefore,
by monotonicity of \(V_\alpha\) under local post-processing in the classical
setting \citep[Theorem~8]{mojahedian2019correlation}, we have
\(V_\alpha(A;R\,|\,C=c)\le V_\alpha(A;T\,|\,C=c)\). This holds for \(P_C\)-almost
every \(c\), proving \eqref{eq:valpha_dpi}.
\end{proof}



\section{Formal projected and sample-wise \texorpdfstring{$V_\alpha$}{V-alpha} bounds}
\label{app:scalable_bounds}

This appendix collects theorem-valid projected, sample-wise, and binarized
consequences of the conditional supersample bound in
Section~\ref{sec:scaling_high_dim}. All density ratios are interpreted as
Radon--Nikodym derivatives with respect to the relevant marginal law; if the
corresponding absolute-continuity or integrability condition fails, the
associated \(V_\alpha\) quantity is interpreted as \(+\infty\).

\paragraph{Local likelihood-ratio deviation.}
For random variables \(A,R,C\) and \(\alpha\ge 1\), define the local
\(V_\alpha\) likelihood-ratio deviation
\begin{align}
\lambda_\alpha(r;A\mid C=c)
\coloneqq
\left(
\mathbb E_{A\sim P_{A\mid C=c}}
\left|
\frac{P_{R\mid A,C}}{P_{R\mid C}}(r\mid A,C=c)-1
\right|^\alpha
\right)^{1/\alpha}.
\label{eq:lambda_alpha_def}
\end{align}
Then, by definition,
\begin{align}
V_\alpha(A;R\mid C=c)
=
\mathbb E_{R\sim P_{R\mid C=c}}
\bigl[
\lambda_\alpha(R;A\mid C=c)
\bigr].
\label{eq:valpha_via_lambda}
\end{align}
When there is no conditioning variable, we write \(\lambda_\alpha(r;A)\) and
\(V_\alpha(A;R)\). Throughout, \(\alpha'\) denotes the Hölder conjugate of
\(\alpha\), with \(\alpha'=\infty\) when \(\alpha=1\).

\paragraph{Supersample notation.}
Recall the supersample construction from Section~\ref{sec:scaling_high_dim}:
\[
\tilde S=\{(\tilde Z_{i,0},\tilde Z_{i,1})\}_{i=1}^n,
\qquad
U=(U_1,\dots,U_n)\sim \mathrm{Unif}(\{0,1\}^n),
\]
with \(U\) independent of \(\tilde S\), and
\(S_i=\tilde Z_{i,U_i}\), \(\bar S_i=\tilde Z_{i,1-U_i}\).
The selected-versus-ghost gap is
\begin{align}
\hat\Delta_{\tilde S}(w,U)
\coloneqq
\frac1n\sum_{i=1}^n
\Bigl[
\ell(w,\tilde Z_{i,1-U_i})
-
\ell(w,\tilde Z_{i,U_i})
\Bigr].
\label{eq:delta_hat_appendix}
\end{align}

\paragraph{Projected statistics.}
Write \(\tilde Z_{i,b}=(\tilde X_{i,b},\tilde Y_{i,b})\). When the loss factors
through predictions, i.e.,
\(\ell(w,(x,y))=\rho(f(w,x),y)\),
we use the supersample prediction vector
\begin{align}
T_{\tilde S}
\coloneqq
\bigl(f(W,\tilde X_{i,b})\bigr)_{i\in[n],\,b\in\{0,1\}}.
\label{eq:T_tildeS_def}
\end{align}
For loss-space reductions, define evaluated losses and their derived statistics:
\begin{align}
L_i^b \coloneqq \ell(W,\tilde Z_{i,b}),
\qquad
D_i \coloneqq L_i^1-L_i^0,
\qquad
P_i \coloneqq (L_i^0,L_i^1).
\label{eq:loss_space_defs}
\end{align}
We also use the loss-difference \emph{vector}
\(D_{\tilde S}\coloneqq (D_1,\dots,D_n)\).
All conditional laws such as \(P_{D_i\mid U_i,\tilde S}\) are understood after
marginalizing over \(U_{-i}\) and any algorithmic randomness.

\paragraph{Binarized (finite-state) losses.}
Assume first that \(0\le \ell(w,z)\le 1\). For each \(i\) and \(b\in\{0,1\}\),
draw \(R_i^b\sim \mathrm{Unif}[0,1]\) independently conditional on
\((L_i^0,L_i^1)\) and independent of \((U,\tilde S)\), and define
\begin{align}
\bar L_i^b \coloneqq \mathbf 1\{R_i^b \le L_i^b\},
\qquad
\bar D_i \coloneqq \bar L_i^1-\bar L_i^0\in\{-1,0,1\},
\qquad
\bar P_i \coloneqq (\bar L_i^0,\bar L_i^1)\in\{0,1\}^2.
\label{eq:binarized_defs}
\end{align}
Then \(\bar L_i^b\mid L_i^b \sim \mathrm{Bern}(L_i^b)\) and
\(\mathbb E[\bar L_i^b\mid L_i^b]=L_i^b\).
For losses bounded in \([0,M]\), we apply the same construction to
\(\tilde L_i^b\coloneqq L_i^b/M\in[0,1]\) and multiply the final bounds by \(M\).

\paragraph{Roadmap.}
We now state and prove Theorems~\ref{th:conditional_output_valpha}--\ref{th:samplewise_binarized_loss_pair_valpha} in the order summarized in
Tables~\ref{tab:holder_instantiations}, \ref{tab:holder_instantiations_samplewise}, and \ref{tab:holder_instantiations_binarized}.

\begin{revision}
These results are not separate bounding principles: they are instantiations
of the same centered $V_\alpha$--H\"older decoupling argument. They differ in
the statistic to which the argument is applied and, where appropriate, in
the use of projection/data processing, sample-wise decomposition, or
binarization to obtain lower-dimensional dependence quantities.
\end{revision}
\newcommand{\laa}{\lambda_\alpha}

\begin{table}[t]
\centering
\rowcolors{2}{gray!10}{white}
\begin{tabularx}{\textwidth}{@{}p{0.22\textwidth}p{0.23\textwidth}X@{}}
\toprule
\textbf{Result} & \textbf{Statistic} & \textbf{Certified bound on $|\overline{\mathrm{gen}}|$} \\
\midrule
Thm.~\ref{th:conditional_output_valpha} &
Output vector $T_{\tilde S}$ &
$\displaystyle \mathbb E_{\tilde S}\!\left[\esssup_t\|\hat\Delta_{\tilde S}(t,U)\|_{\alpha',U}\,V_\alpha(U;T_{\tilde S}\mid\tilde S)\right]$ \\
Thm.~\ref{th:conditional_loss_difference_valpha} &
Loss-difference vector $D_{\tilde S}$ &
$\displaystyle \mathbb E_{\tilde S}\!\left[\esssup_d\left\|\frac1n\sum_i(1-2U_i)d_i\right\|_{\alpha',U}V_\alpha(U;D_{\tilde S}\mid\tilde S)\right]$ \\
\bottomrule
\end{tabularx}
\caption{\rev{Full-vector conditional $V_\alpha$ bounds. These are direct instantiations of the centered $V_\alpha$--H\"older decoupling principle.}}
\label{tab:holder_instantiations}
\end{table}

\begin{table}[t]
\centering
\rowcolors{2}{gray!10}{white}
\begin{tabularx}{\textwidth}{@{}p{0.22\textwidth}p{0.23\textwidth}X@{}}
\toprule
\textbf{Result} & \textbf{Statistic} & \textbf{Certified bound on $|\overline{\mathrm{gen}}|$} \\
\midrule
Thm.~\ref{th:samplewise_loss_difference_valpha} &
$D_i=L_i^1-L_i^0$ &
$\displaystyle \frac1n\sum_i\mathbb E_{\tilde S}\mathbb E_{D_i\mid\tilde S}[|D_i|\laa(D_i;U_i\mid\tilde S)]$; if $0\le\ell\le M$, at most $\displaystyle \frac{M}{n}\sum_i\mathbb E_{\tilde S}V_\alpha(U_i;D_i\mid\tilde S)$ \\
Thm.~\ref{th:samplewise_single_loss_valpha} &
Single loss $L_i^0$ &
$\displaystyle \frac{2}{n}\sum_i\mathbb E_{L_i^0}[|L_i^0-c_i|\laa(L_i^0;U_i)]$; if $0\le\ell\le M$ and $c_i=M/2$, at most $\displaystyle \frac{M}{n}\sum_iV_\alpha(U_i;L_i^0)$ \\
Thm.~\ref{th:samplewise_loss_pair_valpha} &
$P_i=(L_i^0,L_i^1)$ &
$\displaystyle \frac1n\sum_i\mathbb E_{\tilde S}\mathbb E_{P_i\mid\tilde S}[|L_i^1-L_i^0|\laa(P_i;U_i\mid\tilde S)]$; if $0\le\ell\le M$, at most $\displaystyle \frac{M}{n}\sum_i\mathbb E_{\tilde S}V_\alpha(U_i;P_i\mid\tilde S)$ \\
\bottomrule
\end{tabularx}
\caption{\rev{Continuous sample-wise loss-space \(V_\alpha\) generalization bounds. For each result, the table gives the projected statistic and the corresponding certified bound on \(\left|\overline{\mathrm{gen}}\right|\). The weighted expressions retain the local loss magnitude, while the simplified expressions apply when the loss is bounded as \(0\le \ell\le M\).}}
\label{tab:holder_instantiations_samplewise}
\end{table}

\begin{table}[t]
\centering
\rowcolors{2}{gray!10}{white}
\begin{tabularx}{\textwidth}{@{}p{0.22\textwidth}p{0.23\textwidth}X@{}}
\toprule
\textbf{Result} & \textbf{Statistic} & \textbf{Certified bound on $|\overline{\mathrm{gen}}|$ for $0\le\ell\le1$} \\
\midrule
Thm.~\ref{th:samplewise_binarized_single_loss_valpha} &
$\bar L_i^0\in\{0,1\}$ &
$\displaystyle \frac1n\sum_i V_\alpha(U_i;\bar L_i^0)$ \\
Thm.~\ref{th:samplewise_binarized_loss_difference_valpha} &
$\bar D_i\in\{-1,0,1\}$ &
$\displaystyle \frac1n\sum_i\mathbb E_{\bar D_i}[|\bar D_i|\laa(\bar D_i;U_i)]\le\frac1n\sum_iV_\alpha(U_i;\bar D_i)$ \\
Thm.~\ref{th:samplewise_binarized_loss_pair_valpha} &
$\bar P_i\in\{0,1\}^2$ &
$\displaystyle \frac1n\sum_i\mathbb E_{\bar P_i}[|\bar L_i^1-\bar L_i^0|\laa(\bar P_i;U_i)]\le\frac1n\sum_iV_\alpha(U_i;\bar P_i)$ \\
\bottomrule
\end{tabularx}
\caption{\rev{Binarized sample-wise $V_\alpha$ bounds. For losses in $[0,M]$, multiply the displayed right-hand sides by $M$.}}
\label{tab:holder_instantiations_binarized}
\end{table}
\noindent\textbf{Remark (data processing).}
If \(R=g(W,\tilde S)\) is any statistic computed from \((W,\tilde S)\), then
conditionally on \(\tilde S\) we have \(U\to W\to R\), and the conditional
data-processing inequality implies
\(V_\alpha(U;R\mid \tilde S)\le V_\alpha(U;W\mid \tilde S)\) (Appendix~\ref{app:scalable_bounds_lemmas}).

\subsection{Conditional output-space bound}

Write \(\tilde Z_{i,b}=(\tilde X_{i,b},\tilde Y_{i,b})\). Suppose the loss
factors through predictions: there exist measurable maps
\(f:\mathcal W\times\mathcal X\to\mathcal A\) and
\(\rho:\mathcal A\times\mathcal Y\to\mathbb R\) such that
\[
    \ell(w,(x,y))=\rho(f(w,x),y).
\]
For fixed \(\tilde S\), define the supersample prediction vector
\begin{align}
    T_{\tilde S}
    \coloneqq
    \big(f(W,\tilde X_{i,b})\big)_{i\in[n],\,b\in\{0,1\}}.
    \label{eq:app_prediction_vector}
\end{align}
For \(t=(t_{i,b})_{i,b}\), define
\begin{align}
    \hat\Delta_{\tilde S}(t,U)
    \coloneqq
    \frac1n\sum_{i=1}^n
    \left[
        \rho(t_{i,1-U_i},\tilde Y_{i,1-U_i})
        -
        \rho(t_{i,U_i},\tilde Y_{i,U_i})
    \right].
    \label{eq:app_output_gap}
\end{align}
Then
\[
    \hat\Delta_{\tilde S}(W,U)
    =
    \hat\Delta_{\tilde S}(T_{\tilde S},U).
\]

\begin{theorem}[Conditional output-space \(V_\alpha\) bound]
\label{th:conditional_output_valpha}
Suppose the loss factors through predictions as above. For any learning
algorithm and any measurable loss for which the displayed quantities are
well-defined, for every \(\alpha\ge 1\),
\begin{align}
    \left|
        \overline{\mathrm{gen}}(\mu,P_{W|S})
    \right|
    \le
    \mathbb E_{\tilde S}
    \left[
        \esssup_{t\in\mathcal T_{\tilde S}}
        \left\|
            \hat\Delta_{\tilde S}(t,U)
        \right\|_{\alpha',U}
        \,
        V_\alpha(U;T_{\tilde S}|\tilde S)
    \right],
    \label{eq:conditional_output_valpha_bound}
\end{align}
where \(\mathcal T_{\tilde S}\) is the range of
\(w\mapsto(f(w,\tilde X_{i,b}))_{i,b}\), and
\begin{align}
    V_\alpha(U;T_{\tilde S}|\tilde S)
    =
    \mathbb E_{T\sim P_{T|\tilde S}}
    \left[
        \left(
            \mathbb E_U
            \left|
                \frac{P_{T|U,\tilde S}}{P_{T|\tilde S}}
                (T|U,\tilde S)
                -1
            \right|^\alpha
        \right)^{1/\alpha}
    \right].
    \label{eq:conditional_output_valpha_def}
\end{align}
Moreover, conditionally on \(\tilde S\),
\[
    U\to W\to T_{\tilde S}
\]
is a Markov chain, and therefore
\begin{align}
    V_\alpha(U;T_{\tilde S}|\tilde S)
    \le
    V_\alpha(U;W|\tilde S).
    \label{eq:conditional_output_valpha_dpi}
\end{align}
\end{theorem}

\begin{proof}
Apply Lemma~\ref{lem:master_holder} with
\[
    C=\tilde S,\qquad
    A=U,\qquad
    T=T_{\tilde S},\qquad
    \phi_{\tilde S}(t,U)=\hat\Delta_{\tilde S}(t,U).
\]
For every fixed \(t\) and \(\tilde S\), the map
\(U\mapsto\hat\Delta_{\tilde S}(t,U)\) is centered under the uniform selector:
each pair \((\tilde Z_{i,0},\tilde Z_{i,1})\) contributes opposite signs under
\(U_i=0\) and \(U_i=1\). Hence
\[
    \mathbb E_U[\hat\Delta_{\tilde S}(t,U)]=0.
\]
The supremum form of Lemma~\ref{lem:master_holder} gives
\eqref{eq:conditional_output_valpha_bound}. Finally,
\(T_{\tilde S}\) is a measurable function of \((W,\tilde S)\), so the
conditional Markov chain \(U\to W\to T_{\tilde S}\) holds given \(\tilde S\).
Lemma~\ref{lem:valpha_dpi} gives
\eqref{eq:conditional_output_valpha_dpi}.
\end{proof}

\subsection{Conditional loss-difference vector bound}

For fixed \(\tilde S\), define the loss-difference vector
\begin{align}
    D_{\tilde S}
    =
    (D_1,\dots,D_n),
    \qquad
    D_i
    =
    \ell(W,\tilde Z_{i,1})-\ell(W,\tilde Z_{i,0}).
    \label{eq:app_loss_difference_vector}
\end{align}
For \(d=(d_1,\dots,d_n)\in\mathbb R^n\), define
\begin{align}
    \hat\Delta_{\tilde S}(d,U)
    =
    \frac1n\sum_{i=1}^n(1-2U_i)d_i.
    \label{eq:app_loss_difference_gap}
\end{align}
Then
\[
    \hat\Delta_{\tilde S}(W,U)
    =
    \hat\Delta_{\tilde S}(D_{\tilde S},U).
\]

\begin{theorem}[Conditional loss-difference \(V_\alpha\) bound]
\label{th:conditional_loss_difference_valpha}
For any learning algorithm and any measurable loss for which the displayed
quantities are well-defined, for every \(\alpha\ge 1\),
\begin{align}
    \left|
        \overline{\mathrm{gen}}(\mu,P_{W|S})
    \right|
    \le
    \mathbb E_{\tilde S}
    \left[
        \esssup_{d\in\mathcal D_{\tilde S}}
        \left\|
            \hat\Delta_{\tilde S}(d,U)
        \right\|_{\alpha',U}
        \,
        V_\alpha(U;D_{\tilde S}|\tilde S)
    \right],
    \label{eq:conditional_loss_difference_valpha_bound}
\end{align}
where \(\mathcal D_{\tilde S}\) is the range of
\[
    w\mapsto
    \big(\ell(w,\tilde Z_{i,1})-\ell(w,\tilde Z_{i,0})\big)_{i=1}^n,
\]
and
\begin{align}
    V_\alpha(U;D_{\tilde S}|\tilde S)
    =
    \mathbb E_{D\sim P_{D|\tilde S}}
    \left[
        \left(
            \mathbb E_U
            \left|
                \frac{P_{D|U,\tilde S}}{P_{D|\tilde S}}
                (D|U,\tilde S)
                -1
            \right|^\alpha
        \right)^{1/\alpha}
    \right].
    \label{eq:conditional_loss_difference_valpha_def}
\end{align}
If the loss factors through predictions, then \(D_{\tilde S}\) is a measurable
function of \((T_{\tilde S},\tilde S)\), and consequently
\begin{align}
    V_\alpha(U;D_{\tilde S}|\tilde S)
    \le
    V_\alpha(U;T_{\tilde S}|\tilde S).
    \label{eq:loss_difference_data_processing}
\end{align}
\end{theorem}

\begin{proof}
Apply Lemma~\ref{lem:master_holder} with
\[
    C=\tilde S,\qquad
    A=U,\qquad
    T=D_{\tilde S},\qquad
    \phi_{\tilde S}(d,U)=\frac1n\sum_{i=1}^n(1-2U_i)d_i.
\]
For fixed \(d\) and \(\tilde S\),
\[
    \mathbb E_U
    \left[
        \frac1n\sum_{i=1}^n(1-2U_i)d_i
    \right]
    =
    \frac1n\sum_{i=1}^n d_i\,\mathbb E[1-2U_i]
    =
    0.
\]
Thus the centering condition of Lemma~\ref{lem:master_holder} holds, and its
supremum form gives \eqref{eq:conditional_loss_difference_valpha_bound}. If the
loss factors through predictions, then \(D_{\tilde S}\) is a measurable
post-processing of \(T_{\tilde S}\) given \(\tilde S\), so
\eqref{eq:loss_difference_data_processing} follows from
Lemma~\ref{lem:valpha_dpi}.
\end{proof}

\subsection{Sample-wise loss-difference bound}

For each \(i\), define
\begin{align}
    D_i
    =
    \ell(W,\tilde Z_{i,1})-\ell(W,\tilde Z_{i,0}).
    \label{eq:app_samplewise_loss_difference}
\end{align}
The conditional laws of \(D_i\) given \((U_i,\tilde S)\) are understood after
marginalizing over \(U_{-i}\) and any algorithmic randomness.

\begin{theorem}[Sample-wise loss-difference \(V_\alpha\) bound]
\label{th:samplewise_loss_difference_valpha}
For any learning algorithm and any measurable loss for which the displayed
quantities are well-defined, for every \(\alpha\ge 1\),
\begin{align}
    \left|
        \overline{\mathrm{gen}}(\mu,P_{W|S})
    \right|
    \le
    \frac1n\sum_{i=1}^n
    \mathbb E_{\tilde S}
    \mathbb E_{D_i|\tilde S}
    \left[
        |D_i|\,
        \lambda_\alpha(D_i;U_i|\tilde S)
    \right].
    \label{eq:samplewise_loss_difference_valpha_bound}
\end{align}
Equivalently,
\begin{align}
    \lambda_\alpha(D_i;U_i|\tilde S)
    =
    \left(
        \mathbb E_{U_i}
        \left|
            \frac{P_{D_i|U_i,\tilde S}}{P_{D_i|\tilde S}}
            (D_i|U_i,\tilde S)
            -1
        \right|^\alpha
    \right)^{1/\alpha}.
\end{align}
In particular, if \(0\le \ell(w,z)\le M\), then
\begin{align}
    \left|
        \overline{\mathrm{gen}}(\mu,P_{W|S})
    \right|
    \le
    \frac{M}{n}\sum_{i=1}^n
    \mathbb E_{\tilde S}
    \left[
        V_\alpha(U_i;D_i|\tilde S)
    \right].
    \label{eq:samplewise_loss_difference_valpha_bounded}
\end{align}
\end{theorem}

\begin{proof}
The supersample gap decomposes as
\[
    \hat\Delta_{\tilde S}(W,U)
    =
    \frac1n\sum_{i=1}^n(1-2U_i)D_i.
\]
Therefore, by the triangle inequality,
\[
    \left|\overline{\mathrm{gen}}(\mu,P_{W|S})\right|
    \le
    \frac1n\sum_{i=1}^n
    \left|
        \mathbb E[(1-2U_i)D_i]
    \right|.
\]
For each \(i\), apply Lemma~\ref{lem:master_holder} with
\[
    C=\tilde S,\qquad
    A=U_i,\qquad
    T=D_i,\qquad
    \phi_{\tilde S}(d,U_i)=(1-2U_i)d.
\]
For fixed \(d\), the centering condition holds because
\[
    \mathbb E_{U_i}[(1-2U_i)d]=0.
\]
Moreover,
\[
    \|(1-2U_i)d\|_{\alpha',U_i}=|d|.
\]
The weighted form of Lemma~\ref{lem:master_holder} gives
\eqref{eq:samplewise_loss_difference_valpha_bound}. If
\(0\le \ell\le M\), then \(|D_i|\le M\), yielding
\eqref{eq:samplewise_loss_difference_valpha_bounded}.
\end{proof}

\subsection{Sample-wise single-loss bound}

Define
\begin{align}
    L_i^b
    =
    \ell(W,\tilde Z_{i,b}),
    \qquad b\in\{0,1\}.
    \label{eq:app_samplewise_single_loss_Lib}
\end{align}

\begin{theorem}[Sample-wise single-loss \(V_\alpha\) bound]
\label{th:samplewise_single_loss_valpha}
For any learning algorithm and any measurable loss for which the displayed
quantities are well-defined, for every \(\alpha\ge 1\) and every choice of
constants \(c_1,\dots,c_n\in\mathbb R\),
\begin{align}
    \left|
        \overline{\mathrm{gen}}(\mu,P_{W|S})
    \right|
    \le
    \frac{2}{n}\sum_{i=1}^n
    \mathbb E_{L_i^0}
    \left[
        |L_i^0-c_i|\,
        \lambda_\alpha(L_i^0;U_i)
    \right],
    \label{eq:samplewise_single_loss_valpha_weighted}
\end{align}
where
\begin{align}
    \lambda_\alpha(L_i^0;U_i)
    =
    \left(
        \mathbb E_{U_i}
        \left|
            \frac{P_{L_i^0|U_i}}{P_{L_i^0}}
            (L_i^0|U_i)
            -1
        \right|^\alpha
    \right)^{1/\alpha}.
\end{align}
In particular, if \(0\le \ell(w,z)\le M\), then choosing \(c_i=M/2\) gives
\begin{align}
    \left|
        \overline{\mathrm{gen}}(\mu,P_{W|S})
    \right|
    \le
    \frac{M}{n}\sum_{i=1}^n
    V_\alpha(U_i;L_i^0).
    \label{eq:samplewise_single_loss_valpha_bounded}
\end{align}
\end{theorem}

\begin{proof}
For each \(i\), define the expected \(i\)-th supersample contribution
\[
    G_i
    =
    \mathbb E[L_i^{1-U_i}-L_i^{U_i}].
\]
By the exchangeability of
\((\tilde Z_{i,0},\tilde Z_{i,1})\) and the symmetry of \(U_i\),
\[
    G_i
    =
    \mathbb E[L_i^0|U_i=1]
    -
    \mathbb E[L_i^0|U_i=0].
\]
Since \(U_i\sim\mathrm{Bern}(1/2)\), for any constant \(c_i\),
\begin{align}
    G_i
    &=
    2\mathbb E[(2U_i-1)(L_i^0-c_i)].
    \label{eq:single_loss_symmetry_identity}
\end{align}
Apply Lemma~\ref{lem:master_holder} without a conditioning variable, with
\[
    A=U_i,\qquad
    T=L_i^0,\qquad
    \phi(\ell_0,U_i)=2(2U_i-1)(\ell_0-c_i).
\]
For fixed \(\ell_0\),
\[
    \mathbb E_{U_i}[2(2U_i-1)(\ell_0-c_i)]=0,
\]
and
\[
    \|2(2U_i-1)(\ell_0-c_i)\|_{\alpha',U_i}
    =
    2|\ell_0-c_i|.
\]
Summing over \(i\) proves
\eqref{eq:samplewise_single_loss_valpha_weighted}. If \(0\le L_i^0\le M\) and
\(c_i=M/2\), then \(2|L_i^0-M/2|\le M\), which proves
\eqref{eq:samplewise_single_loss_valpha_bounded}.
\end{proof}

\subsection{Sample-wise loss-pair bound}

Define the evaluated loss pair
\begin{align}
    P_i
    =
    (L_i^0,L_i^1)
    =
    \big(\ell(W,\tilde Z_{i,0}),\ell(W,\tilde Z_{i,1})\big).
    \label{eq:app_loss_pair}
\end{align}
For \(p=(a,b)\), define the local selected-versus-ghost gap
\begin{align}
    \delta_i(p,U_i)
    \coloneqq
    a\mathbf 1\{U_i=1\}
    +b\mathbf 1\{U_i=0\}
    -a\mathbf 1\{U_i=0\}
    -b\mathbf 1\{U_i=1\}.
    \label{eq:app_delta_pair}
\end{align}
Equivalently,
\[
    \delta_i((a,b),U_i)=(1-2U_i)(b-a).
\]

\begin{theorem}[Sample-wise loss-pair \(V_\alpha\) bound]
\label{th:samplewise_loss_pair_valpha}
For any learning algorithm and any measurable loss for which the displayed
quantities are well-defined, for every \(\alpha\ge 1\),
\begin{align}
    \left|
        \overline{\mathrm{gen}}(\mu,P_{W|S})
    \right|
    \le
    \frac1n\sum_{i=1}^n
    \mathbb E_{\tilde S}
    \mathbb E_{P_i|\tilde S}
    \left[
        \|\delta_i(P_i,U_i)\|_{\alpha',U_i}
        \,
        \lambda_\alpha(P_i;U_i|\tilde S)
    \right].
    \label{eq:samplewise_loss_pair_valpha_weighted}
\end{align}
Since
\[
    \|\delta_i(P_i,U_i)\|_{\alpha',U_i}
    =
    |L_i^1-L_i^0|,
\]
this is equivalently
\begin{align}
    \left|
        \overline{\mathrm{gen}}(\mu,P_{W|S})
    \right|
    \le
    \frac1n\sum_{i=1}^n
    \mathbb E_{\tilde S}
    \mathbb E_{P_i|\tilde S}
    \left[
        |L_i^1-L_i^0|
        \,
        \lambda_\alpha(P_i;U_i|\tilde S)
    \right].
    \label{eq:samplewise_loss_pair_valpha_simplified}
\end{align}
In particular, if \(0\le \ell(w,z)\le M\), then
\begin{align}
    \left|
        \overline{\mathrm{gen}}(\mu,P_{W|S})
    \right|
    \le
    \frac{M}{n}\sum_{i=1}^n
    \mathbb E_{\tilde S}
    \left[
        V_\alpha(U_i;P_i|\tilde S)
    \right].
    \label{eq:samplewise_loss_pair_valpha_bounded}
\end{align}
\end{theorem}

\begin{proof}
The \(i\)-th supersample contribution satisfies
\[
    L_i^{1-U_i}-L_i^{U_i}
    =
    \delta_i(P_i,U_i).
\]
For fixed \(p=(a,b)\),
\[
    \mathbb E_{U_i}[\delta_i(p,U_i)]
    =
    \frac12(b-a)+\frac12(a-b)
    =
    0.
\]
Apply Lemma~\ref{lem:master_holder} with
\[
    C=\tilde S,\qquad
    A=U_i,\qquad
    T=P_i,\qquad
    \phi_{\tilde S}(p,U_i)=\delta_i(p,U_i).
\]
This gives \eqref{eq:samplewise_loss_pair_valpha_weighted}. Since
\[
    |\delta_i((a,b),0)|=|\delta_i((a,b),1)|=|b-a|,
\]
we have
\[
    \|\delta_i(P_i,U_i)\|_{\alpha',U_i}=|L_i^1-L_i^0|,
\]
which gives \eqref{eq:samplewise_loss_pair_valpha_simplified}. If
\(0\le\ell\le M\), then \(|L_i^1-L_i^0|\le M\), yielding
\eqref{eq:samplewise_loss_pair_valpha_bounded}.
\end{proof}

\subsection{Binarized single-loss bound}

Assume first that \(0\le \ell(w,z)\le 1\). For each \(i\), let
\[
    L_i^0=\ell(W,\tilde Z_{i,0}).
\]
Draw \(R_i^0\sim\mathrm{Unif}[0,1]\) independently conditional on \(L_i^0\), and
define
\begin{align}
    \bar L_i^0
    =
    \mathbf 1\{R_i^0\le L_i^0\}.
    \label{eq:app_binarized_single_loss}
\end{align}
Then
\[
    \bar L_i^0|L_i^0\sim\mathrm{Bern}(L_i^0),
    \qquad
    \mathbb E[\bar L_i^0|L_i^0]=L_i^0.
\]

\begin{theorem}[Sample-wise binarized single-loss \(V_\alpha\) bound]
\label{th:samplewise_binarized_single_loss_valpha}
Suppose \(0\le \ell(w,z)\le 1\). For any learning algorithm and any
\(\alpha\ge 1\),
\begin{align}
    \left|
        \overline{\mathrm{gen}}(\mu,P_{W|S})
    \right|
    \le
    \frac1n\sum_{i=1}^n
    V_\alpha(U_i;\bar L_i^0).
    \label{eq:samplewise_binarized_single_loss_valpha}
\end{align}
More generally, if \(0\le \ell(w,z)\le M\), then binarizing \(L_i^0/M\) gives
\begin{align}
    \left|
        \overline{\mathrm{gen}}(\mu,P_{W|S})
    \right|
    \le
    \frac{M}{n}\sum_{i=1}^n
    V_\alpha(U_i;\bar L_i^0).
    \label{eq:samplewise_binarized_single_loss_valpha_M}
\end{align}
\end{theorem}

\begin{proof}
By the single-loss symmetry identity
\eqref{eq:single_loss_symmetry_identity}, with \(c_i=1/2\),
\[
    \mathbb E[L_i^{1-U_i}-L_i^{U_i}]
    =
    2\mathbb E[(2U_i-1)(L_i^0-1/2)].
\]
Because \(\mathbb E[\bar L_i^0|L_i^0]=L_i^0\) and the binarization noise is
independent of \(U_i\) conditional on \(L_i^0\),
\[
    \mathbb E[L_i^{1-U_i}-L_i^{U_i}]
    =
    2\mathbb E[(2U_i-1)(\bar L_i^0-1/2)].
\]
Apply Lemma~\ref{lem:master_holder} without conditioning, with
\[
    A=U_i,\qquad
    T=\bar L_i^0,\qquad
    \phi(\bar \ell,U_i)=2(2U_i-1)(\bar \ell-1/2).
\]
For fixed \(\bar\ell\in\{0,1\}\),
\[
    \mathbb E_{U_i}[\phi(\bar\ell,U_i)]=0,
    \qquad
    \|\phi(\bar\ell,U_i)\|_{\alpha',U_i}=1.
\]
Thus the \(i\)-th contribution is bounded by
\(V_\alpha(U_i;\bar L_i^0)\). Summing over \(i\) gives
\eqref{eq:samplewise_binarized_single_loss_valpha}. For losses in \([0,M]\),
apply the same argument to \(L_i^0/M\) and multiply the final bound by \(M\).
\end{proof}

Since \(\bar L_i^0\) is obtained from \(L_i^0\) by independent post-processing,
the Markov chain
\[
    U_i\to L_i^0\to \bar L_i^0
\]
implies, by Lemma~\ref{lem:valpha_dpi},
\begin{align}
    V_\alpha(U_i;\bar L_i^0)
    \le
    V_\alpha(U_i;L_i^0).
    \label{eq:binarized_single_loss_dpi}
\end{align}

\subsection{Binarized loss-difference bound}

Assume first that \(0\le \ell(w,z)\le 1\). For \(b\in\{0,1\}\), define
\[
    L_i^b=\ell(W,\tilde Z_{i,b}).
\]
Draw \(R_i^b\sim\mathrm{Unif}[0,1]\) independently conditional on
\((L_i^0,L_i^1)\), and define
\begin{align}
    \bar L_i^b
    =
    \mathbf 1\{R_i^b\le L_i^b\},
    \qquad
    \bar D_i
    =
    \bar L_i^1-\bar L_i^0.
    \label{eq:app_binarized_loss_difference}
\end{align}
Then
\[
    \mathbb E[\bar L_i^b|L_i^b]=L_i^b,
    \qquad
    \bar D_i\in\{-1,0,1\}.
\]

\begin{theorem}[Sample-wise binarized loss-difference \(V_\alpha\) bound]
\label{th:samplewise_binarized_loss_difference_valpha}
Suppose \(0\le \ell(w,z)\le 1\). For any learning algorithm and any
\(\alpha\ge 1\),
\begin{align}
    \left|
        \overline{\mathrm{gen}}(\mu,P_{W|S})
    \right|
    \le
    \frac1n\sum_{i=1}^n
    \mathbb E_{\bar D_i}
    \left[
        |\bar D_i|
        \,
        \lambda_\alpha(\bar D_i;U_i)
    \right].
    \label{eq:samplewise_binarized_loss_difference_weighted}
\end{align}
In particular, since \(|\bar D_i|\le 1\),
\begin{align}
    \left|
        \overline{\mathrm{gen}}(\mu,P_{W|S})
    \right|
    \le
    \frac1n\sum_{i=1}^n
    V_\alpha(U_i;\bar D_i).
    \label{eq:samplewise_binarized_loss_difference_valpha}
\end{align}
More generally, if \(0\le \ell(w,z)\le M\), then binarizing \(L_i^b/M\) gives
\begin{align}
    \left|
        \overline{\mathrm{gen}}(\mu,P_{W|S})
    \right|
    \le
    \frac{M}{n}\sum_{i=1}^n
    V_\alpha(U_i;\bar D_i).
    \label{eq:samplewise_binarized_loss_difference_valpha_M}
\end{align}
\end{theorem}

\begin{proof}
Let \(D_i=L_i^1-L_i^0\). The \(i\)-th supersample contribution is
\[
    L_i^{1-U_i}-L_i^{U_i}
    =
    (1-2U_i)D_i.
\]
By construction,
\[
    \mathbb E[\bar D_i|L_i^0,L_i^1]
    =
    \mathbb E[\bar L_i^1-\bar L_i^0|L_i^0,L_i^1]
    =
    L_i^1-L_i^0
    =
    D_i.
\]
Since the binarization noise is independent of \(U_i\) conditional on
\((L_i^0,L_i^1)\),
\[
    \mathbb E[(1-2U_i)D_i]
    =
    \mathbb E[(1-2U_i)\bar D_i].
\]
Apply Lemma~\ref{lem:master_holder} without conditioning, with
\[
    A=U_i,\qquad
    T=\bar D_i,\qquad
    \phi(d,U_i)=(1-2U_i)d.
\]
For fixed \(d\in\{-1,0,1\}\),
\[
    \mathbb E_{U_i}[(1-2U_i)d]=0,
    \qquad
    \|(1-2U_i)d\|_{\alpha',U_i}=|d|.
\]
This gives \eqref{eq:samplewise_binarized_loss_difference_weighted}. The
simplified bound \eqref{eq:samplewise_binarized_loss_difference_valpha} follows
from \(|\bar D_i|\le1\). The case \(0\le\ell\le M\) follows by applying the
same argument to \(L_i^b/M\) and multiplying by \(M\).
\end{proof}

Let
\[
    \bar P_i=(\bar L_i^0,\bar L_i^1),
    \qquad
    P_i=(L_i^0,L_i^1).
\]
Since \(\bar D_i=\bar L_i^1-\bar L_i^0\) is a deterministic function of
\(\bar P_i\), and \(\bar P_i\) is an independent post-processing of \(P_i\),
Lemma~\ref{lem:valpha_dpi} gives
\begin{align}
    V_\alpha(U_i;\bar D_i)
    \le
    V_\alpha(U_i;\bar P_i)
    \le
    V_\alpha(U_i;P_i).
    \label{eq:binarized_difference_dpi_chain}
\end{align}

\subsection{Binarized loss-pair bound}

Assume first that \(0\le \ell(w,z)\le 1\). For \(b\in\{0,1\}\), define
\[
    L_i^b=\ell(W,\tilde Z_{i,b}).
\]
Draw \(R_i^b\sim\mathrm{Unif}[0,1]\) independently conditional on
\((L_i^0,L_i^1)\), and define
\begin{align}
    \bar L_i^b
    =
    \mathbf 1\{R_i^b\le L_i^b\},
    \qquad
    \bar P_i
    =
    (\bar L_i^0,\bar L_i^1).
    \label{eq:app_binarized_loss_pair}
\end{align}
Then \(\bar P_i\in\{0,1\}^2\). For
\(\bar p=(a,b)\in\{0,1\}^2\), define
\begin{align}
    \bar\delta_i(\bar p,U_i)
    \coloneqq
    a\mathbf 1\{U_i=1\}
    +b\mathbf 1\{U_i=0\}
    -a\mathbf 1\{U_i=0\}
    -b\mathbf 1\{U_i=1\}.
    \label{eq:app_binarized_delta_pair}
\end{align}
Equivalently,
\[
    \bar\delta_i((a,b),U_i)=(1-2U_i)(b-a).
\]

\begin{theorem}[Sample-wise binarized loss-pair \(V_\alpha\) bound]
\label{th:samplewise_binarized_loss_pair_valpha}
Suppose \(0\le \ell(w,z)\le 1\). For any learning algorithm and any
\(\alpha\ge 1\),
\begin{align}
    \left|
        \overline{\mathrm{gen}}(\mu,P_{W|S})
    \right|
    \le
    \frac1n\sum_{i=1}^n
    \mathbb E_{\bar P_i}
    \left[
        \|\bar\delta_i(\bar P_i,U_i)\|_{\alpha',U_i}
        \,
        \lambda_\alpha(\bar P_i;U_i)
    \right].
    \label{eq:samplewise_binarized_loss_pair_weighted}
\end{align}
Since
\[
    \|\bar\delta_i(\bar P_i,U_i)\|_{\alpha',U_i}
    =
    |\bar L_i^1-\bar L_i^0|,
\]
this is equivalently
\begin{align}
    \left|
        \overline{\mathrm{gen}}(\mu,P_{W|S})
    \right|
    \le
    \frac1n\sum_{i=1}^n
    \mathbb E_{\bar P_i}
    \left[
        |\bar L_i^1-\bar L_i^0|
        \,
        \lambda_\alpha(\bar P_i;U_i)
    \right].
    \label{eq:samplewise_binarized_loss_pair_simplified}
\end{align}
In particular, since \(|\bar L_i^1-\bar L_i^0|\le1\),
\begin{align}
    \left|
        \overline{\mathrm{gen}}(\mu,P_{W|S})
    \right|
    \le
    \frac1n\sum_{i=1}^n
    V_\alpha(U_i;\bar P_i).
    \label{eq:samplewise_binarized_loss_pair_valpha}
\end{align}
More generally, if \(0\le \ell(w,z)\le M\), then binarizing \(L_i^b/M\) gives
\begin{align}
    \left|
        \overline{\mathrm{gen}}(\mu,P_{W|S})
    \right|
    \le
    \frac{M}{n}\sum_{i=1}^n
    V_\alpha(U_i;\bar P_i).
    \label{eq:samplewise_binarized_loss_pair_valpha_M}
\end{align}
\end{theorem}

\begin{proof}
By construction,
\[
    \mathbb E[\bar L_i^b|L_i^b]=L_i^b,
    \qquad b\in\{0,1\}.
\]
Hence
\[
    \mathbb E[\bar L_i^1-\bar L_i^0|L_i^0,L_i^1]
    =
    L_i^1-L_i^0.
\]
Since the binarization noise is independent of \(U_i\) conditional on
\((L_i^0,L_i^1)\),
\[
    \mathbb E[L_i^{1-U_i}-L_i^{U_i}]
    =
    \mathbb E[\bar\delta_i(\bar P_i,U_i)].
\]
Apply Lemma~\ref{lem:master_holder} without conditioning, with
\[
    A=U_i,\qquad
    T=\bar P_i,\qquad
    \phi(\bar p,U_i)=\bar\delta_i(\bar p,U_i).
\]
For fixed \(\bar p=(a,b)\),
\[
    \mathbb E_{U_i}[\bar\delta_i(\bar p,U_i)]
    =
    \frac12(b-a)+\frac12(a-b)
    =
    0.
\]
Therefore Lemma~\ref{lem:master_holder} gives
\eqref{eq:samplewise_binarized_loss_pair_weighted}. Since
\[
    |\bar\delta_i((a,b),0)|
    =
    |\bar\delta_i((a,b),1)|
    =
    |b-a|,
\]
we have
\[
    \|\bar\delta_i(\bar P_i,U_i)\|_{\alpha',U_i}
    =
    |\bar L_i^1-\bar L_i^0|,
\]
which gives \eqref{eq:samplewise_binarized_loss_pair_simplified}. The simplified
finite-state bound \eqref{eq:samplewise_binarized_loss_pair_valpha} follows
from \(|\bar L_i^1-\bar L_i^0|\le1\). For losses in \([0,M]\), apply the same
argument to \(L_i^b/M\) and multiply the resulting bound by \(M\).
\end{proof}

Since \(\bar D_i=\bar L_i^1-\bar L_i^0\) is a deterministic function of
\(\bar P_i\), and \(\bar P_i\) is an independent post-processing of
\(P_i=(L_i^0,L_i^1)\), Lemma~\ref{lem:valpha_dpi} gives
\begin{align}
    V_\alpha(U_i;\bar D_i)
    \le
    V_\alpha(U_i;\bar P_i)
    \le
    V_\alpha(U_i;P_i).
    \label{eq:binarized_pair_dpi_chain}
\end{align}


\begin{revision}
\subsection{Conditional coordinatewise versions used in the experiments}
\label{sec:conditional_samplewise_experimental}

\begin{lemma}[Fair-binary-selector $V_\alpha$ invariance]
\label{lem:fair_binary_valpha}
Let $U_i\sim\mathrm{Bern}(1/2)$ be independent of $C$, and let
$p_u(t\mid c)$ be the conditional density or mass of $T$ given $U_i=u,C=c$.
For every $\alpha\ge1$,
\begin{align}
 \lambda_\alpha(t;U_i\mid C=c)
 &=\frac{|p_0(t\mid c)-p_1(t\mid c)|}{p_0(t\mid c)+p_1(t\mid c)},
 \label{eq:fair_binary_local_lambda}\\
 V_\alpha(U_i;T\mid C=c)
 &=\frac12\int |p_0(t\mid c)-p_1(t\mid c)|\,dt.
 \label{eq:fair_binary_valpha}
\end{align}
The local factor is defined to be zero where both densities vanish; the
integral is a sum for discrete $T$. Multiplying its integrand by any
nonnegative statistic weight $w(t)$ gives
$\mathbb E_{T\mid c}[w(T)\lambda_\alpha(T;U_i\mid c)]$, which is also
independent of $\alpha$.
\end{lemma}
\begin{proof}
The marginal is $p(t\mid c)=(p_0+p_1)/2$. The two centered density ratios
are $(p_0-p_1)/(p_0+p_1)$ and its negative, so their $L_\alpha$ norm under
the fair selector is the absolute value of either one. Integrating against
$p(t\mid c)$ proves both identities and the weighted statement.
\end{proof}

\begin{corollary}[Conditional variants of T9 and T11--T13]
\label{cor:conditional_samplewise_valpha}
Under the supersample protocol, let $0\le\ell\le M$ and let $c_i$ be fixed
constants. The conditional single-loss and binarized bounds used in the
experiments are
\begin{align}
 |\overline{\mathrm{gen}}|
 &\le \frac1n\sum_i\mathbb E_{\tilde S}\sum_{b=0}^1
 \mathbb E_{L_i^b\mid\tilde S}
 [|L_i^b-c_i|\lambda_\alpha(L_i^b;U_i\mid\tilde S)],
 &&\text{(T9)}\label{eq:t9_conditional_weighted}\\
 |\overline{\mathrm{gen}}|
 &\le \frac{M}{2n}\sum_i\mathbb E_{\tilde S}
 \sum_{b=0}^1 V_\alpha(U_i;\bar L_i^b\mid\tilde S),
 &&\text{(T11)}\label{eq:t11_conditional_bounded}\\
 |\overline{\mathrm{gen}}|
 &\le \frac{M}{n}\sum_i\mathbb E_{\tilde S}
 \mathbb E_{\bar D_i\mid\tilde S}
 [|\bar D_i|\lambda_\alpha(\bar D_i;U_i\mid\tilde S)],
 &&\text{(T12)}\label{eq:t12_conditional_weighted}\\
 |\overline{\mathrm{gen}}|
 &\le \frac{M}{n}\sum_i\mathbb E_{\tilde S}
 \mathbb E_{\bar P_i\mid\tilde S}
 [|\bar L_i^1-\bar L_i^0|\lambda_\alpha(\bar P_i;U_i\mid\tilde S)],
 &&\text{(T13)}\label{eq:t13_conditional_weighted}
\end{align}
Here $\bar L_i^b\mid L_i^b\sim\mathrm{Bern}(L_i^b/M)$, with independent
binarization noise, $\bar D_i=\bar L_i^1-\bar L_i^0$, and
$\bar P_i=(\bar L_i^0,\bar L_i^1)$. Taking $c_i=M/2$ and replacing
$|L_i^b-c_i|$ by $M/2$ yields the bounded T9 form.
T8 and T10 already have conditional forms in
Theorems~\ref{th:samplewise_loss_difference_valpha} and
\ref{th:samplewise_loss_pair_valpha}.
\end{corollary}
\begin{proof}
At fixed $\tilde S=s$, decompose the expected coordinate gap as
$\mathbb E[(2U_i-1)(L_i^0-c_i)\mid s]
 +\mathbb E[(1-2U_i)(L_i^1-c_i)\mid s]$.
Apply Lemma~\ref{lem:master_holder} separately to each term with $A=U_i$,
$C=\tilde S$, and $T=L_i^b$. The centered sign has absolute value one,
so its loss-weighted $L_{\alpha'}$ norm is $|L_i^b-c_i|$.
The triangle inequality and averaging over $s,i$ give T9.
Independent binarization preserves the expected gap after scaling by $M$.
The same two-sided argument with center $1/2$ gives T11; applying the
lemma to $M(1-2U_i)\bar D_i$ or
$M(1-2U_i)(\bar L_i^1-\bar L_i^0)$ gives T12 and T13.
\end{proof}
\end{revision}

\section{Experimental Details and Methodology}
\label{app:simulation}

\rev{This appendix documents the conditional, coordinatewise plug-in
estimators in Figures~\ref{fig:ratio_to_error}--\ref{fig:mnist_noise_ratio}.
We augment the BinaryJS pipeline of \citet{dong2025exactly} with the
sample-wise and binarized $V_\alpha$ quantities in
Theorems~\ref{th:samplewise_loss_difference_valpha}--\ref{th:samplewise_binarized_loss_pair_valpha}
and the conditional variants in
Corollary~\ref{cor:conditional_samplewise_valpha}. All methods use the same
supersamples, selector masks, trained models, and loss arrays.}

\paragraph{Supersample protocol.}
For each sample size
\[
    n\in\{75,250,1000,4000\},
\]
we draw a supersample
\[
    \tilde S=\{(\tilde Z_{i,0},\tilde Z_{i,1})\}_{i=1}^n.
\]
A selector vector
\[
    U=(U_1,\dots,U_n)\in\{0,1\}^n
\]
selects the training sample
\[
    S_U=\{\tilde Z_{i,U_i}\}_{i=1}^n,
\]
and the complementary examples
\[
    S_{\bar U}=\{\tilde Z_{i,1-U_i}\}_{i=1}^n
\]
form the ghost sample. For each \(n\), we use \(5\) independent supersample
seeds. For each fixed supersample seed, we use \(30\) independent selector
seeds. Thus, each fixed supersample yields \(30\) independently trained models,
one for each selector vector.
\paragraph{Dataset and preprocessing.}
The experiment uses the MNIST handwritten digit dataset~\citep{lecun1998mnist},
restricted to the digits \(4\) and \(9\) for classification task with
\(\texttt{num\_classes}=2\). After filtering MNIST to labels \(4\) and \(9\),
the dataset wrapper maps the labels to the binary labels used by the experiment.
Inputs are transformed using \texttt{ToTensor()} and normalized with mean
\(0.456\) and standard deviation \(0.224\). No data augmentation, ImageNet
resizing, or label noise is used.

\paragraph{Model architecture.}
The learning algorithm is a convolutional neural network trained on \(S_U\).
The CNN has four convolutional layers followed by two fully connected layers:
\[
\begin{aligned}
&\mathrm{Conv}(1,32,k=4,s=2,p=1) + \rev{\mathrm{BatchNorm} + \mathrm{ReLU}},\\
&\mathrm{Conv}(32,32,k=4,s=2,p=1) + \rev{\mathrm{BatchNorm} + \mathrm{ReLU}},\\
&\mathrm{Conv}(32,64,k=3,s=2,p=0) + \rev{\mathrm{BatchNorm} + \mathrm{ReLU}},\\
&\mathrm{Conv}(64,256,k=3,s=1,p=0) + \rev{\mathrm{BatchNorm} + \mathrm{ReLU}},\\
&\mathrm{Flatten},\quad
\mathrm{FC}(256,128) + \mathrm{ReLU},\quad
\mathrm{FC}(128,2).
\end{aligned}
\]
Here \(\mathrm{Conv}(c_{\mathrm{in}},c_{\mathrm{out}},k,s,p)\) denotes a
convolutional layer with \(c_{\mathrm{in}}\) input channels,
\(c_{\mathrm{out}}\) output channels, kernel size \(k\), stride \(s\), and
padding \(p\).

\paragraph{Optimization and determinism.}
The optimizer is Adam with fixed learning rate \(10^{-3}\), batch size \(128\),
weight decay \(0\), and no learning-rate schedule. Models are trained for
\(200\) epochs using standard, unclipped cross-entropy. We evaluate the final
checkpoint, corresponding to \(\texttt{epoch199.mdl}\). 
Deterministic mode is enabled, with zero data-loading workers and fixed training
seed \(42\).

\paragraph{Losses and empirical generalization gap.}
For each trained model \(W\), we evaluate predictions on the entire supersample
of \(2n\) points. From these predictions we compute both the \(0\)-\(1\) loss
and the clipped cross-entropy loss\rev{, defined for evaluation as
$\ell_{\mathrm{CE}}(w,z)=\min\{-\log p_w(y\mid x),1\}$, so that $M=1$
for both losses}. For each pair \(i\), define
\[
    L_i^b=\ell(W,\tilde Z_{i,b}),
    \qquad b\in\{0,1\}.
\]
The empirical training and ghost losses are
\[
    L_{\mathrm{train}}
    =
    \frac1n\sum_{i=1}^n L_i^{U_i},
    \qquad
    L_{\mathrm{ghost}}
    =
    \frac1n\sum_{i=1}^n L_i^{1-U_i}.
\]
The empirical selected-versus-ghost gap is
\[
    \widehat{\mathrm{gap}}
    =
    L_{\mathrm{ghost}}-L_{\mathrm{train}}
    =
    \frac1n\sum_{i=1}^n
    \left(
        L_i^{1-U_i}-L_i^{U_i}
    \right).
\]
\rev{For each fixed supersample seed $s$, we average the selected-versus-ghost
gap over its selector runs to obtain $\widehat{\mathrm{gap}}_s$, and denote
the coordinatewise-averaged increment by $\widehat B_s$. The figures plot
the mean of $\widehat B_s/|\widehat{\mathrm{gap}}_s|$ over supersample seeds,
with SEM across those seeds. Ratio $1$ denotes agreement with the observed
gap; ratios require a nonzero denominator.}

\paragraph{Bounds included in the main experiment.}
\rev{For each fixed supersample, we estimate a distinct conditional table at
each coordinate from its $30$ MNIST (or $40$ CIFAR-10) selector runs.
We evaluate each contribution before averaging across coordinates; pooling
coordinates before evaluating $V_\alpha$ would target a different quantity.
The six plotted curves use the bounded form for T8, the weighted forms for
T9, T10, T12, and T13, and the two-sided binarized form for T11.}

\begin{itemize}
    \item \textbf{Theorem~\ref{th:samplewise_loss_difference_valpha}:}
    the sample-wise loss-difference bound, based on
    \[
        D_i=L_i^1-L_i^0.
    \]
    This is denoted as T8 in the figure.

    \item \textbf{Theorem~\ref{th:samplewise_single_loss_valpha}:}
    \rev{the conditional two-sided single-loss bound in
    Corollary~\ref{cor:conditional_samplewise_valpha}, using both $L_i^0$
    and $L_i^1$ with center $c_i=1/2$. This is denoted T9.}

    \item \textbf{Theorem~\ref{th:samplewise_loss_pair_valpha}:}
    the sample-wise loss-pair bound, based on
    \[
        P_i=(L_i^0,L_i^1).
    \]
    This is denoted as T10.

    \item \textbf{Theorem~\ref{th:samplewise_binarized_single_loss_valpha}:}
    the sample-wise binarized single-loss bound. For losses in \([0,1]\), we
    define
    \[
        \bar L_i^b\mid L_i^b\sim\mathrm{Bern}(L_i^b).
    \]
    \rev{The two-sided conditional form in
    Corollary~\ref{cor:conditional_samplewise_valpha} is denoted T11.}

    \item \textbf{Theorem~\ref{th:samplewise_binarized_loss_difference_valpha}:}
    the sample-wise binarized loss-difference bound, based on
    \[
        \bar D_i=\bar L_i^1-\bar L_i^0\in\{-1,0,1\}.
    \]
    This is denoted as T12.

    \item \textbf{Theorem~\ref{th:samplewise_binarized_loss_pair_valpha}:}
    the sample-wise binarized loss-pair bound, based on
    \[
        \bar P_i=(\bar L_i^0,\bar L_i^1)\in\{0,1\}^2.
    \]
    This is denoted as T13.
\end{itemize}


\paragraph{Empirical \(V_\alpha\) estimation.}
\begin{revision}
For fixed supersample $s$ and coordinate $i$, we estimate
$\hat p_i(t\mid u,s)$ by relative frequencies among selector runs with
$U_i=u$. Both selector values must be observed. We use the known prior
$P(U_i=u\mid s)=1/2$, rather than its empirical frequency, and set
$\hat p_i(t\mid s)=\frac12[\hat p_i(t\mid0,s)+\hat p_i(t\mid1,s)]$.
By Lemma~\ref{lem:fair_binary_valpha}, the local factor and estimate are
\[
 \hat\lambda_{\alpha,i}(t\mid s)
 =\frac{|\hat p_i(t\mid0,s)-\hat p_i(t\mid1,s)|}
        {\hat p_i(t\mid0,s)+\hat p_i(t\mid1,s)},
 \qquad
 \widehat V_{\alpha,i}(s)
 =\sum_t\hat p_i(t\mid s)\hat\lambda_{\alpha,i}(t\mid s).
\]
The factor is set to zero on empty states. For a statistic weight $w_i$,
the weighted contribution is
$\sum_t\hat p_i(t\mid s)w_i(t)\hat\lambda_{\alpha,i}(t\mid s)$.
For binned continuous losses, the code averages the original unbinned
magnitude within each bin and selector group, with weight $1/2$ for each
selector group. These are plug-in approximations to the unbinned weighted
quantities; the finite-state binarized variants avoid this discretization.
\end{revision}

\paragraph{Discretization.}
For \(0\)-\(1\) loss, the relevant variables are finite-valued and exact support
coding is used. For continuous clipped cross-entropy quantities, we use
histogram coding with \(30\) bins by default. For scalar quantities, we bin the
scalar values directly. For loss pairs, we apply row coding after binning the
two coordinates. If the empirical support is small, exact support coding is used
automatically.

The binarized bounds do not require Monte Carlo binarization. Instead, we use
the exact expected finite-state joint tables. For Theorem~\ref{th:samplewise_binarized_single_loss_valpha},
we form the expected \(2\times 2\) table for
\((U_i,\bar L_i^b)\). For Theorem~\ref{th:samplewise_binarized_loss_difference_valpha},
we form the expected \(2\times 3\) table for
\((U_i,\bar D_i)\). For Theorem~\ref{th:samplewise_binarized_loss_pair_valpha},
we form the expected \(2\times 4\) table for
\((U_i,\bar P_i)\). This makes the binarized estimators deterministic given the
loss arrays and selectors.

\paragraph{Symmetrization for single-loss bounds.}
\rev{The unconditional one-sided T9/T11 theorems use exchangeability after
averaging over the supersample. At a fixed supersample, we instead use the
conditional two-sided formulas in
Corollary~\ref{cor:conditional_samplewise_valpha}. The implemented estimate
is the half-sum of the two single-side estimates (each carrying the
one-sided theorem's factor of two). This gives the sum of the two weighted
conditional contributions for T9, and $M/2$ times the sum of the two
conditional $V_\alpha$ terms for T11.}

\paragraph{Choice of \(\alpha\).}
\rev{The implementation evaluates
$\alpha\in\{1.005,1.1,1.25,1.5,2.0,3.0,5.0\}$.
For the conditional coordinatewise estimators with the known fair
selector prior, Lemma~\ref{lem:fair_binary_valpha} implies that
all entries give the same value, including for the
magnitude-weighted expressions. We therefore report $\alpha=1.005$
as a fixed representative.
This invariance is specific to the fair binary-selector setting
and does not imply $\alpha$-invariance of the parameter-space or
full-vector quantities studied in the BLR experiments.}

\paragraph{Interpretation of the ratio figure.}
\rev{For $0$--$1$ loss, T9--T13 remain close to ratio $1$, whereas
the bounded T8 curve is more conservative. For clipped cross-entropy,
T8 and T9 yield larger ratios than T10--T13. T8 replaces the local
loss-difference magnitude $|D_i|$ by the worst-case bound $M=1$.
T9 controls the two centered single-loss contributions separately,
which can lose cancellation between the paired losses. In contrast,
T10 retains the paired magnitude $|L_i^1-L_i^0|$, and the binarized
variants use their finite-state constructions. T10--T13 remain close
to the observed gap across the displayed sample sizes and settings.}

\begin{revision}
\subsection{Additional deep-learning experiments}
\label{app:additional_deep_experiments}

\paragraph{MNIST with random label noise.}
We repeat the clean MNIST $4$-vs.-$9$ protocol after independently flipping
the binary labels of both the selected and ghost examples with probability
$\delta\in\{0.05,0.10,0.15\}$, following the noisy settings considered by
\citet{dong2025exactly}. Labels are fixed within each supersample across
selector runs. Figure~\ref{fig:mnist_noise_ratio} shows T9--T13 near ratio
$1$ for $0$--$1$ loss, while T8 has ratios of roughly $1.6$--$1.9$.
For clipped cross-entropy, T10--T13 remain close to the gap at all three
noise levels, while T8 and T9 are more conservative; T9 is not uniformly
tighter than T8, as seen at $\delta=0.05$ and $n=4000$.

\begin{figure}[t]
    \centering
    \includegraphics[width=0.9\linewidth]{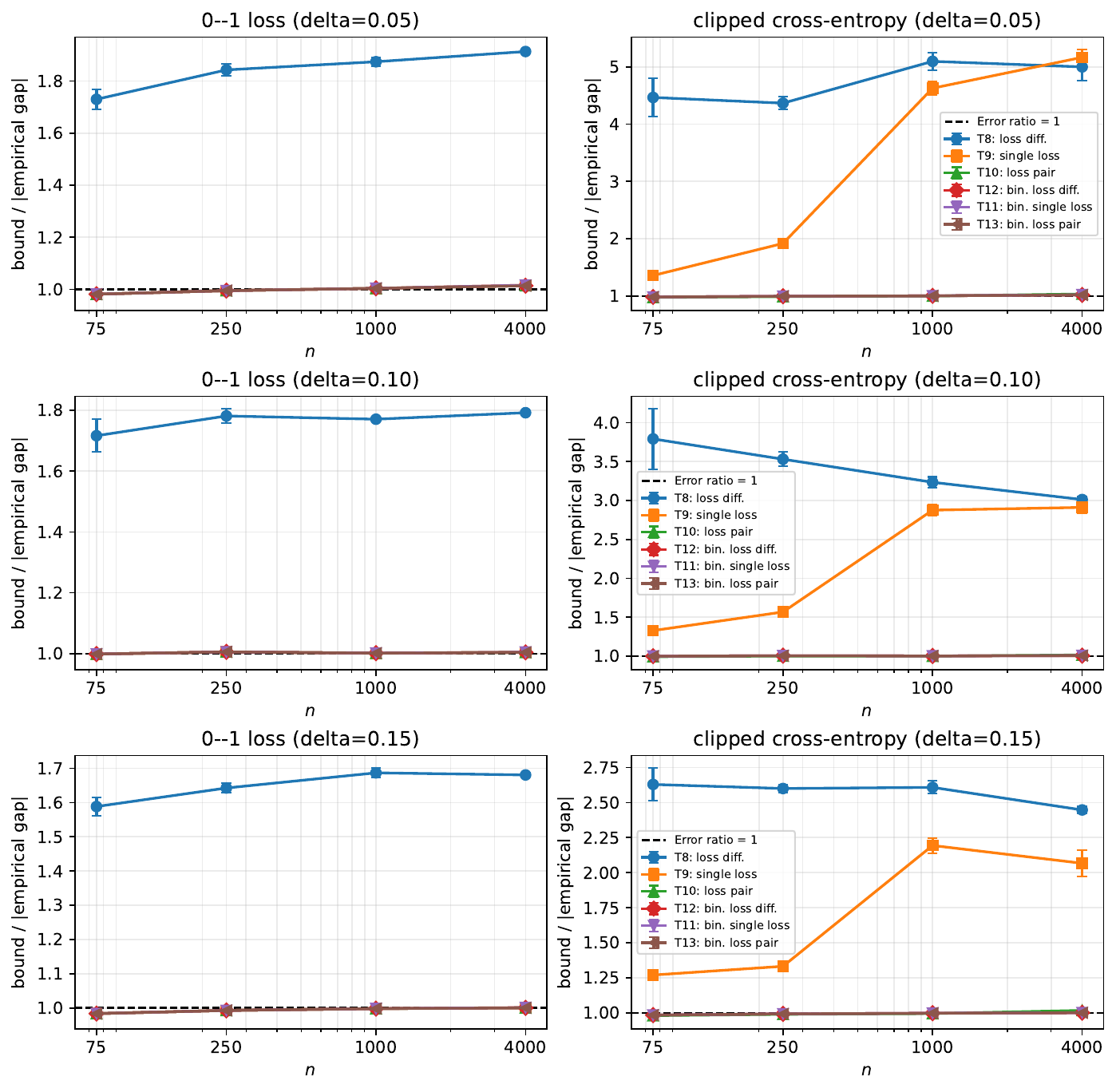}
    \caption{\rev{Additional MNIST $4$-vs.-$9$ results under random label noise.
    Each panel reports the bound-to-empirical-gap ratio for $0$--$1$ loss
    (left) and clipped cross-entropy (right). Rows correspond to
    $\delta=0.05,0.10,0.15$ from top to bottom. Error bars show SEM over
    five supersample seeds.}}
    \label{fig:mnist_noise_ratio}
\end{figure}

\paragraph{CIFAR-10.}
For Figure~\ref{fig:cifar10_resnet_ratio}, we use torchvision ResNet-50
with ImageNet-1K V1 pretrained weights, replace the final layer by a
$10$-class head, and fine-tune all parameters. For each
$n\in\{1000,5000,20000\}$ we use supersample seeds $0,1$ and selector
seeds $0,\ldots,39$. A seed-dependent $80/20$ split of the CIFAR-10 training
set precedes sampling $2n$ points without replacement from the training
partition. Training uses $40$ epochs of SGD with constant learning rate
$0.01$, momentum $0.9$, batch size $64$, weight decay $0$, and
unclipped cross-entropy. We evaluate the final-epoch models.
Training augmentation consists of random horizontal flips and random
$32\times32$ crops with padding $4$. Inputs are converted to tensors,
normalized using the standard ImageNet channel statistics, and then
bilinearly resized to $224\times224$.
All bound and gap evaluations use the same fixed supersample images and
labels with deterministic preprocessing and no random augmentation. The
six estimators, clipping rule, and representative $\alpha$ are as above.
\end{revision}
\section{Additional Results on Bayesian Linear Regression}
\label{app:simulation_blr}

This appendix documents the BLR experiment used
to compare the unconditional parameter-space bounds and the
supersample-based conditional bounds from
Section~\ref{sec:scaling_high_dim}. The discussion here is specific to the
moderate-dimensional run with $d \in \{1,2,5,10\}$ and fixed sample size
$n=100$, and Figures~\ref{fig:blr_extended_conditional} and
\ref{fig:blr_extended_conditional_alpha} report the resulting bounds and the
selected $\alpha^\star$ values.

\subsection{Common BLR Setup}
In this BLR experiment, the data are generated from
\[
    y_i = \mathbf{x}_i^\top \mathbf{w}_{\mathrm{true}} + \epsilon_i,
    \qquad
    \mathbf{x}_i \sim \mathcal{N}(0, I_d),
    \qquad
    \epsilon_i \sim \mathcal{N}(0, \sigma_{\mathrm{noise}}^2),
\]
with $\mathbf{w}_{\mathrm{true}} = 0.8\,\mathbf{1}_d$,
$\sigma_{\mathrm{noise}}^2 = 0.16$, and Gaussian prior
$p(\mathbf{w}) = \mathcal{N}(0, \lambda^{-1} I_d)$ with $\lambda = 2.0$.
Conjugacy gives the posterior
$p(\mathbf{w}\mid S)=\mathcal{N}(\boldsymbol{\mu}_n,\boldsymbol{\Sigma}_n)$,
where
\begin{align}
    \boldsymbol{\Sigma}_n
    &=
    \left(\lambda I_d + \beta \mathbf{X}^\top \mathbf{X}\right)^{-1},
    &
    \boldsymbol{\mu}_n
    &=
    \beta \boldsymbol{\Sigma}_n \mathbf{X}^\top \mathbf{y},
\end{align}
and $\beta = 1/\sigma_{\mathrm{noise}}^2$.

All bounds use the same clipped squared loss
\[
    \ell(\mathbf{w}, z)
    =
    \min\!\left((y-\mathbf{w}^\top \mathbf{x})^2, c\right),
    \qquad c = 4.0.
\]
Since $\ell \in [0,c]$, Hoeffding's lemma gives the loss-scale
\begin{align}
    \sigma_{\mathrm{loss}}^2 = \frac{c^2}{4} = 4.0,
\end{align}
The same clipping level is used throughout, with each bound using its
corresponding theoretical loss prefactor. 
The BLR
implementation enforces the following fairness rules throughout:
\begin{itemize}
    \item MI, CMI, and $V_\alpha$ are evaluated on matched draws from the same
    data-generating process.
    \item The Bayesian prior is not reused as a surrogate for the marginal
    law $p_W$; unconditional marginals are estimated by independent posterior
    mixtures.
    \item In the supersample experiments, the mixture used to estimate
    $p(W\mid \tilde S)$ is built from selections of $U$ that are independent of
    the selections used to evaluate the conditional expectations.
\end{itemize}

\subsection{Monte Carlo Estimators}
\paragraph{Empirical generalization.}
For each sampled dataset, we draw one posterior sample
$\mathbf{W} \sim p(\mathbf{w}\mid S)$, compute the empirical clipped loss on the
training sample, and compare it against a fixed holdout estimate of the
population loss based on $10{,}000$ fresh examples. Reported empirical gaps are
averages over independent datasets, with standard errors computed from the same
replicates.

\paragraph{MI and CMI.}
The MI curve uses the classical bound of \citet{xu2017information},
$\sqrt{2\sigma_{\mathrm{loss}}^2 I(W;S)/n}$, where
$I(W;S)$ is estimated by Monte Carlo from
$\log p(W\mid S)-\log p_W(W)$ and $p_W$ is an independent posterior mixture.
The CMI curve uses the supersample construction of
\citet{steinke2020}:
\[
c\sqrt{\frac{2}{n} I(W;U\mid \tilde S)}.
\]
For each supersample $\tilde S$, the conditional marginal
$p(W\mid \tilde S)=\mathbb{E}_U[p(W\mid S_U)]$ is approximated by a Gaussian
mixture over independent selector draws, and the nested expectations are
evaluated with log-sum-exp stabilization.

\paragraph{Unconditional parameter-space $V_\alpha$.}
In the present experiment, all unconditional parameter-space curves are
evaluated on the grid $d \in \{1,2,5,10\}$ at fixed $n=100$. The
parameter-space profile $V_\alpha(S;W)$ is estimated by posterior-mixture Monte
Carlo for every tested $\alpha$, and the $\alpha=2$ slice is used for the fixed
$V_2$ comparison. The reported parameter-space curves are:
\begin{itemize}
    \item Parameter fixed $V_2$ ($\alpha=2$) (Theorem~\ref{th:variance_bound}),
    \item Parameter Cor.~2 $V_\alpha$
    (Corollary~\ref{cor:subgaussian}), and
    \item Parameter finite-candidate approx.\ $V_\alpha$
    (Theorem~\ref{th-main}), included only as a diagnostic.
\end{itemize}

\paragraph{Conditional $V_\alpha(U;W\mid \tilde S)$.}
The conditional experiment estimates $V_\alpha(U;W\mid \tilde S)$ on the same
supersamples used for CMI. For $\alpha>1$, with
$q=\alpha'=\alpha/(\alpha-1)$, the rigorous theorem-valid prefactor is the
Khinchine prefactor for the supersample Rademacher average,
\[
    \frac{K_q c}{\sqrt{n}},
    \qquad
    K_q =
    \begin{cases}
        1, & 1 \le q \le 2,\\
        \sqrt{q-1}, & q \ge 2.
    \end{cases}
\]
Hence $\alpha=2$ gives the exact familiar prefactor $c/\sqrt{n}$, while the
$\alpha=1$ endpoint is treated separately with the bounded-loss norm
$\|\widehat{\Delta}_{\tilde S}(w,U)\|_{\infty,U} \le c$.

\paragraph{Finite-candidate approximations and split alpha-selection.}
The experiment also reports Conditional finite-candidate approx.\
$V_\alpha$ (Theorem~\ref{th:conditional_valpha}) and Parameter
finite-candidate approx.\ $V_\alpha$ (Theorem~\ref{th-main}), in which the
theorem-side supremum is replaced by a maximum over a finite candidate set
built from posterior means, marginal-mixture samples, prior samples, and
deterministic anchors. These finite-candidate approximation curves are used as
diagnostics for the conservativeness of the analytic prefactor bounding: when
they lie substantially below the corresponding theorem-valid curves, the gap
indicates that much of the looseness comes from the bounded prefactor or
supremum control rather than from the estimated dependence term itself. They
are still not theorem-valid upper bounds unless the finite candidate maximum is
externally certified to dominate the true supremum.

For every $\alpha$-dependent curve in the extended comparison, $\alpha^\star$ is
selected on an independent Monte Carlo pool and then evaluated on a disjoint
pool. The plotted uncertainty bands are standard errors on that evaluation pool
at the selected $\alpha^\star$. This split-selection protocol is used for the
Conditional theorem-valid $V_\alpha$
(Theorem~\ref{th:conditional_valpha}), Conditional finite-candidate approx.\
$V_\alpha$ (Theorem~\ref{th:conditional_valpha}), Parameter Cor.~2 $V_\alpha$
(Corollary~\ref{cor:subgaussian}), and Parameter finite-candidate approx.\
$V_\alpha$ (Theorem~\ref{th-main}).

\subsection{Reported Figures and Interpretation}
Figure~\ref{fig:blr_extended_conditional} shows
the resulting bounds at $n=100$ over
$d \in \{1,2,5,10\}$, while
Figure~\ref{fig:blr_extended_conditional_alpha} reports the selected
$\alpha^\star$ values for the split-selected curves.

\begin{figure*}[t]
    \centering
    \includegraphics[width=0.85\textwidth]{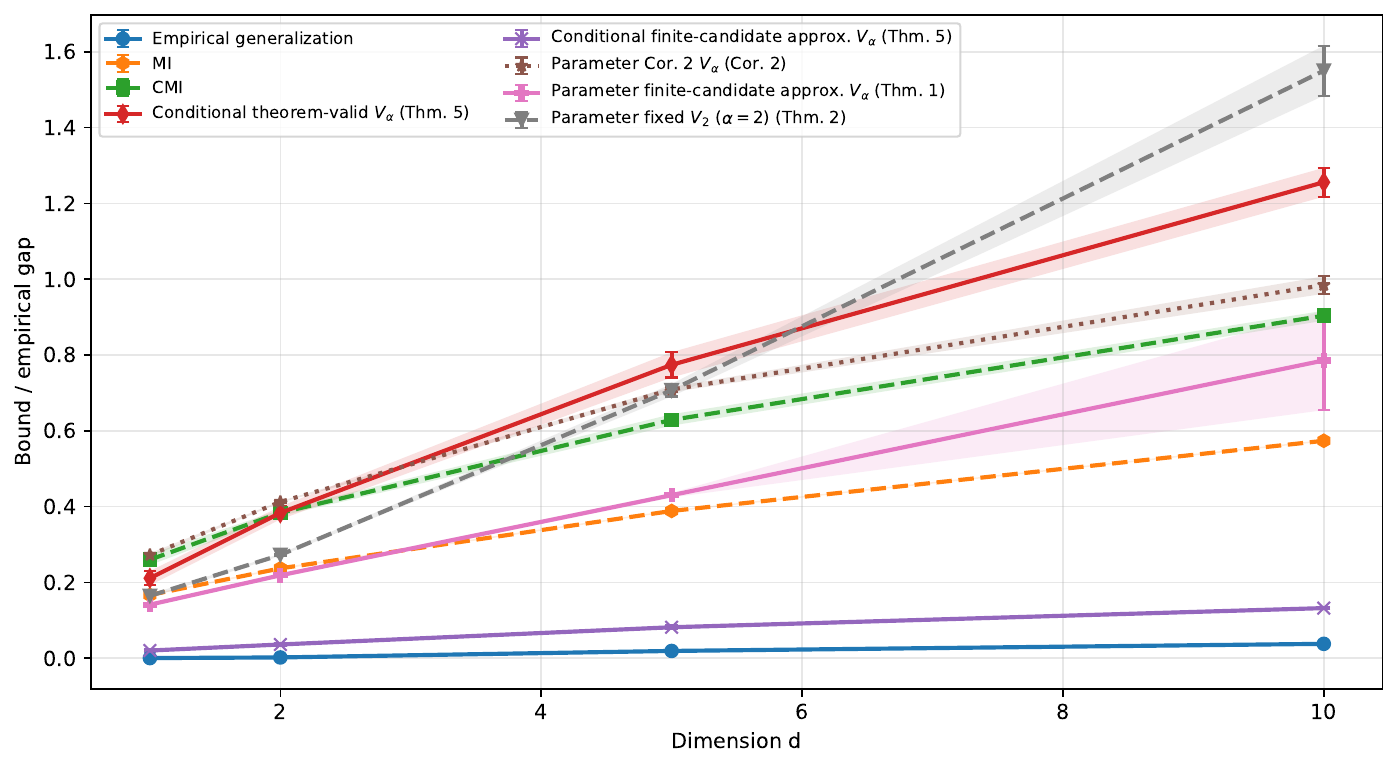}
    \caption{BLR bound comparison for the reported experiment. The figure includes Conditional theorem-valid $V_\alpha$ (Theorem~\ref{th:conditional_valpha}), Conditional finite-candidate approx.\ $V_\alpha$ (Theorem~\ref{th:conditional_valpha}), Parameter Cor.~2 $V_\alpha$ (Corollary~\ref{cor:subgaussian}), Parameter finite-candidate approx.\ $V_\alpha$ (Theorem~\ref{th-main}), and Parameter fixed $V_2$ ($\alpha=2$) (Theorem~\ref{th:variance_bound}), together with the MI and CMI baselines.}
    \label{fig:blr_extended_conditional}
\end{figure*}

\begin{figure*}[t]
    \centering
    \includegraphics[width=0.85\textwidth]{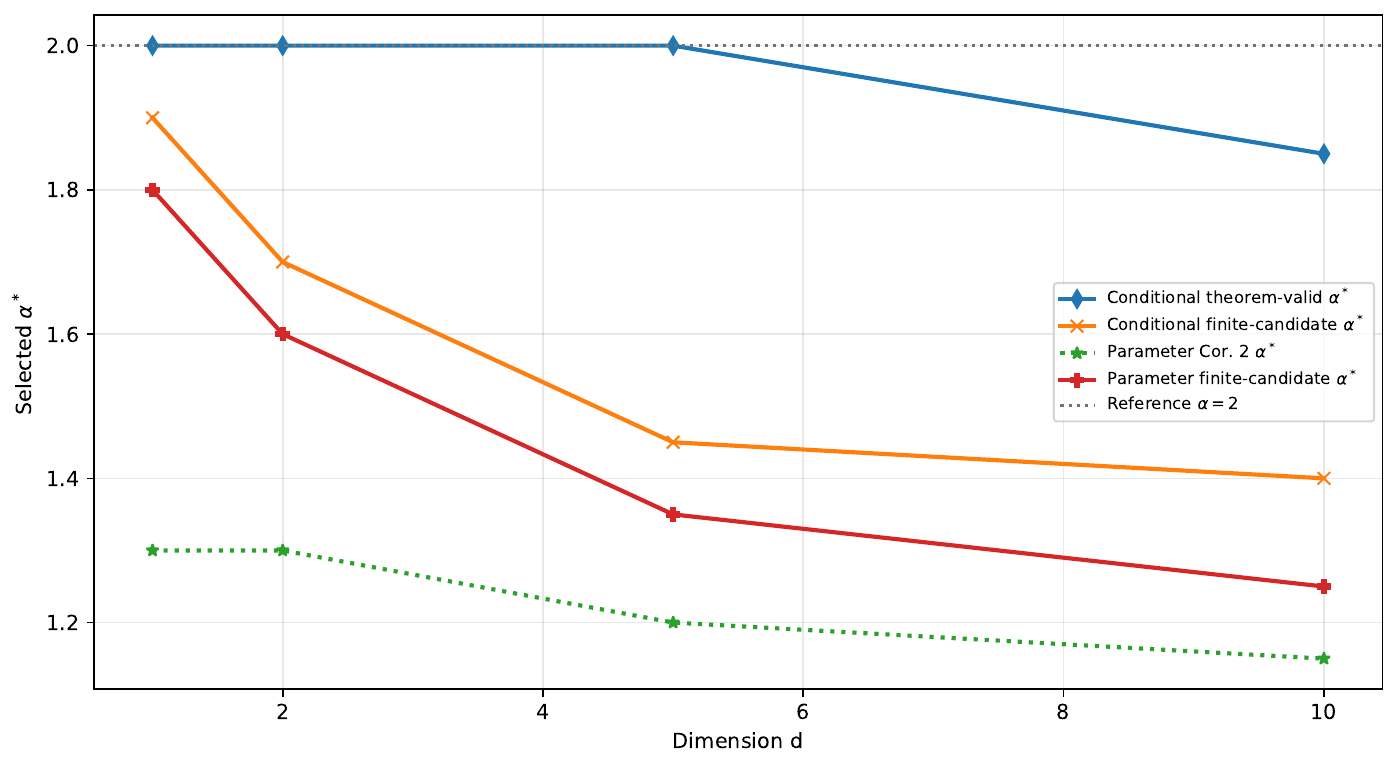}
    \caption{Selected $\alpha^\star$ values for the reported BLR experiment. Conditional theorem-valid $\alpha^\star$ stays at $2$ for $d=1,2,5$ and shifts slightly below $2$ only at $d=10$, whereas Conditional finite-candidate $\alpha^\star$, Parameter Cor.~2 $\alpha^\star$, and Parameter finite-candidate $\alpha^\star$ move below $2$ earlier.}
    \label{fig:blr_extended_conditional_alpha}
\end{figure*}

Several patterns are stable across dimensions. First, the estimated CMI
is smaller than the estimated MI, but the standard bounded-loss CMI bound is
numerically larger because its prefactor is twice that of the Hoeffding-based MI
bound. Second, among the theorem-valid $V_\alpha$ curves, the conditional
formulation is dimension-free but numerically conservative in this regime. Third, the
parameter-space side illustrates the main trade-off behind our framework. The
fixed $V_2$ curve is competitive in the lowest dimension, but it becomes loose
as the parameter dimension grows. Optimizing
the pure Corollary~\ref{cor:subgaussian} bound over $\alpha$ reduces this
high-dimensional growth.

The finite-candidate diagnostics give substantially smaller numerical values than the Monte Carlo estimates of the theorem-based $V_\alpha$ expressions. In particular, the conditional finite-candidate
approximation increases only from about $0.020$ at $d=1$ to about $0.132$ at
$d=10$, and the parameter finite-candidate approximation also improves
considerably over the parameter-space theorem-valid curves in moderate and high
dimension. In our view, this gap is mainly evidence that the theorem-valid
prefactor control is conservative in this BLR regime, rather than evidence that
the underlying dependence term itself is poorly captured. 

The selected $\alpha^\star$ values are shown in
Figure~\ref{fig:blr_extended_conditional_alpha}. This plot clarifies how the interpolation
behaves across dimensions.

The $\alpha^\star$ trajectories support the central interpretation of the
paper. The theorem-valid conditional curve remains close to the variance-like
choice $\alpha=2$ (it selects $2.0, 2.0, 2.0, 1.85$ as $d$ grows), whereas the
conditional finite-candidate approximation shifts to
$1.9, 1.7, 1.45, 1.4$. On the parameter side, the pure Corollary~\ref{cor:subgaussian}
curve selects $1.3, 1.3, 1.2, 1.15$, and the parameter finite-candidate
approximation selects $1.8, 1.6, 1.35, 1.25$. Thus, as the ambient dimension
increases, the best-performing parameter-space formulations consistently move
away from the rigid $\alpha=2$ regime toward smaller $\alpha$. This is exactly
the flexibility that the $V_\alpha$ family is meant to expose.

Table~\ref{tab:blr_appendix_config} summarizes the simulation budgets used for
these final appendix figures.

\begin{table}[t]
\centering
\caption{Configuration used to generate Figures~\ref{fig:blr_extended_conditional}
and~\ref{fig:blr_extended_conditional_alpha}.}
\label{tab:blr_appendix_config}
\begin{tabular}{@{}llc@{}}
\toprule
\textbf{Parameter} & \textbf{Description} & \textbf{Value} \\
\midrule
$\sigma_{\mathrm{noise}}^2$ & Observation noise variance & $0.16$ \\
$\lambda$ & Gaussian prior precision & $2.0$ \\
$c$ & Clipping level in the squared loss & $4.0$ \\
$\sigma_{\mathrm{loss}}^2$ & Subgaussian scaling constant & $4.0$ \\
$\mathbf w_{\mathrm{true}}$ & Ground-truth coefficient in each coordinate & $0.8$ \\
$d$ & Dimensions evaluated & $[1,2,5,10]$ \\
$n$ & Training sample size & $100$ \\
\midrule
$M_{\mathrm{data}}$ & Independent datasets for empirical averages & $256$ \\
$M_{\mathrm{ref}}$ & Reference datasets for mixture estimation & $256$ \\
$N_{\mathrm{holdout}}$ & Holdout size for population loss estimation & $10{,}000$ \\
$N_W$ & Weight samples for parameter-space $V_\alpha$ estimates & $384$ \\
$M_{\mathrm{sup}}$ & Supersamples per conditional estimate & $20$ \\
$M_{\mathrm{sel}}/M_{\mathrm{eval}}$ & Split-selection supersamples & $10/10$ \\
$K_{\mathrm{mix}}$ & Marginal mixture components & $32$ \\
$J_U$ & Selector draws per CMI estimate & $24$ \\
\midrule
$R_{\mathrm{param}}$ & Parameter-profile repeats (selection/evaluation) & $4/4$ \\
$R_{\mathrm{proxy}}$ & Finite-candidate proxy repeats (selection/evaluation) & $4/4$ \\
$N_{\mathrm{cand}}$ & Total parameter proxy candidates & $226$ \\
$\alpha$ grid & Optimization range & $1.05$ to $3.0$ in steps of $0.05$ \\
\bottomrule
\end{tabular}
\end{table}


\section{Additional Results on Canonical Channel Models}
\label{app:canonical_channels}

To complement the continuous Bayesian linear regression experiments in the main text, we provide additional simulations on three classical, discrete channel models: the Binary Symmetric Channel (BSC), the Binary Erasure Channel (BEC), and the Z-Channel. 

These settings are highly instructive because they allow for the \textit{exact analytical computation} of both the true generalization error and the theoretical bounds, entirely eliminating any variance or approximation error from Monte Carlo estimation.

\paragraph{Setup.} 
For these experiments, we consider a minimal learning setting where the dataset size is $n = 1$. The instance space and the hypothesis space are binary: $\mathcal{Z} = \mathcal{W} = \{0, 1\}$.  In our setup, the learning kernel $P_{W|S}$ is chosen to behave like the canonical channel under consideration (BSC, BEC, or Z-channel). Concretely, the channel takes the dataset $S$ as input, assumed to be uniformly distributed over $\mathcal{Z}$, and produces the algorithm’s output $W$ according to the corresponding channel transition probabilities. For the BSC and Z-channel we use binary output support. For the BEC we enlarge the output support to $\{0,1,e\}$, where $e$ denotes the erasure symbol; in the squared-loss comparison we map $e$ to the neutral prediction $0.5$. We define the loss function as the squared error $\ell(w, z) = (w - z)^2$. 
Because both $w$ and $z$ are restricted to $\{0, 0.5, 1\}$, the loss is strictly bounded in $[0, 1]$. By Hoeffding's Lemma, the loss is therefore $\frac{1}{4}$-subgaussian (i.e., $\sigma^2 = 0.25$). For visual clarity in the comparisons below, the constant leading term $\sqrt{\sigma^2/n}$ is factored out and omitted from the plotted curves.

All information quantities in these comparisons use natural logarithms (nats). The results, shown in Figure~\ref{fig:canonical}, plot the performance of the $V_2$ bound against the classical MI bound across the full parameter range of each respective channel.

\begin{figure}[h!]
    \centering
    \includegraphics[width=\linewidth]{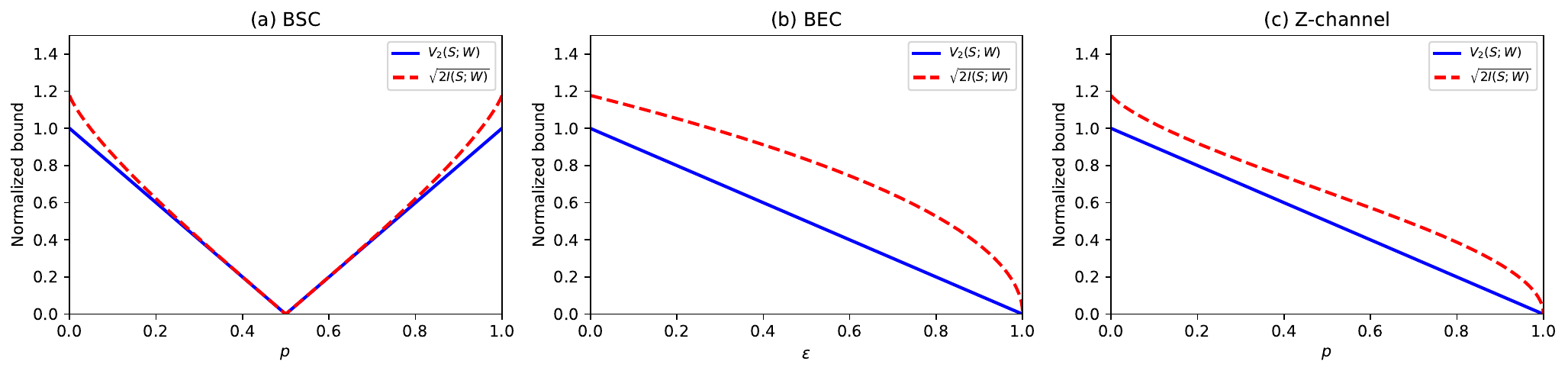}
    \caption{Normalized proposed (blue solid) and mutual-information
    (red dashed) generalization bounds at $\alpha=2$ for (a) the BSC,
    (b) the BEC, and (c) the Z-channel. The common factor
    $\sqrt{\sigma^2/n}=1/2$ is omitted.}
    \label{fig:canonical}
\end{figure}

For the three canonical channels shown here, the proposed $V_2$ bound is no larger than the MI bound and is strictly smaller away from equality cases. These examples illustrate the possible gains in discrete settings.

\subsection{Illustration of the \texorpdfstring{$\alpha$}{alpha} Trade-off}
\label{app:alpha_tradeoff}

We now provide an explicit illustration of the trade-off described in Remark~\ref{remark:tradeoff}. While the previous experiments fix $\alpha = 2$, here we vary $\alpha$ to examine how the two competing terms in Theorem~\ref{th-main} interact.

\paragraph{Setup.}
We consider a simple Z-channel model. The dataset $S$ consists of a single sample drawn from a Bernoulli($q$) distribution with parameter $q=0.8$, and the algorithm output $W$ is generated via a Z-channel with transition probability $p=0.3$. As in the previous experiments, we use the squared loss $\ell(w,z) = (w - z)^2$.
\enlargethispage{2\baselineskip}
\paragraph{Trade-off behavior.}
Figure~\ref{fig:trade-off} plots the two components of our bound as functions of $\alpha$: the information term $V_\alpha(S;W)$ and the loss sensitivity term given by the $L_{\alpha'}$-norm. As predicted by Remark~\ref{remark:tradeoff}, these terms exhibit opposing monotonic behavior: $V_\alpha(S;W)$ increases with $\alpha$, while the loss sensitivity term decreases due to the conjugacy relation between $\alpha$ and $\alpha'$.

Their product, which constitutes our generalization bound, is also shown in the figure. Notably, the bound is minimized at an intermediate value of $\alpha$ (in this instance, $\alpha = 2$), illustrating that the tightest guarantee is achieved by balancing these two effects rather than optimizing either term individually.


\rev{The minimum plotted $V_\alpha$ bound (approximately $0.224$ at $\alpha\approx2$) is below the MI bound (approximately $0.314$). The two curves can cross away from the minimizing $\alpha$.}
\begin{figure}[h]
    \centering
    \includegraphics[width=0.85\linewidth]{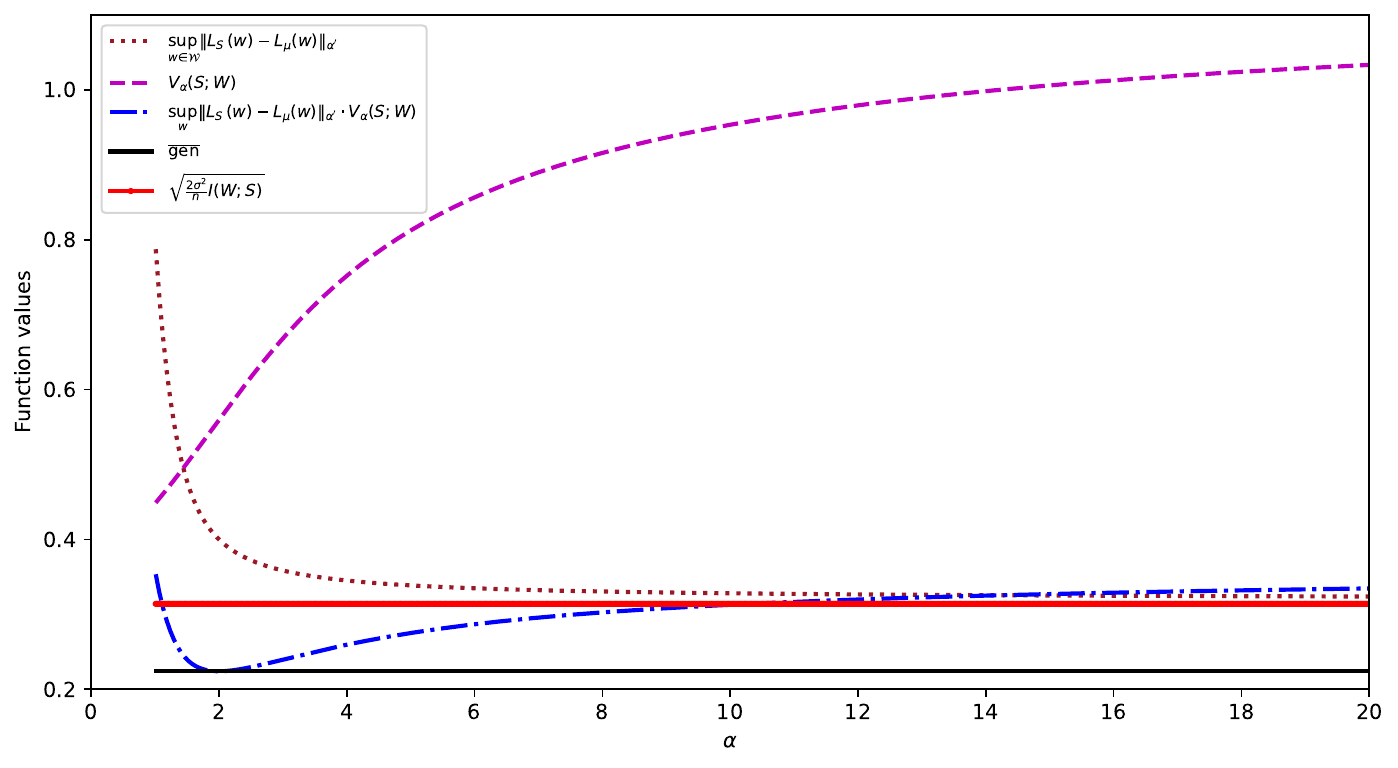}
    \caption{Illustration of the trade-off in Remark~\ref{remark:tradeoff}. As $\alpha$ increases, the information term $V_\alpha(S;W)$ grows while the loss sensitivity term $\|\cdot\|_{\alpha'}$ shrinks. Their product (our bound) is minimized at an intermediate value, providing a tighter guarantee than the classical mutual information bound.}
    \label{fig:trade-off}
\end{figure}

This example illustrates the trade-off described in Remark~\ref{remark:tradeoff}: balancing the loss-sensitivity and dependence terms yields the smallest bound at an intermediate value of $\alpha$ in the considered channel setting.


\end{document}